\documentclass{article}

\usepackage{fullpage}
\usepackage{cite}

\usepackage{booktabs}       
\usepackage{amsfonts}      
\usepackage{nicefrac}       
\usepackage{microtype}

\usepackage{amsmath,amssymb,amscd,amstext,amsthm,amsfonts}
\usepackage{latexsym,float,graphics,color,epsfig,euscript,subfigure,wrapfig,ifthen}
\usepackage{extarrows}
\usepackage{stmaryrd}

\usepackage{romannum}
 
\usepackage{paralist,graphicx}
\usepackage{cite,url,hyperref}
\usepackage{tikz,pgfplots}
\usetikzlibrary{calc,matrix,arrows,angles,patterns, quotes,intersections, 3d}

\usepackage{relsize}

\theoremstyle{plain}
\newtheorem{theorem}{Theorem}[section]

\newtheorem{proposition}[theorem]{Proposition}
\newtheorem{corollary}[theorem]{Corollary}

\theoremstyle{definition}
\newtheorem{definition}[theorem]{Definition}
\newtheorem{example}[theorem]{Example}
\newtheorem{algorithm}[theorem]{Algorithm}

\theoremstyle{remark}
\newtheorem{remark}{Remark}

\newcommand{\mbbN}{{\mathbb N}}

\newcommand{\mbbR}{{\mathbb R}}

\newcommand{\mcalG}{{\mathcal G}}

\newcommand{\mbbRminmax}{\overline{\overline{\mbbR}}}

\newcommand{\va}{{\mathbf a}}
\newcommand{\vb}{{\mathbf b}}

\newcommand{\vu}{{\mathbf u}}
\newcommand{\vv}{{\mathbf v}}

\newcommand{\vx}{{\mathbf x}}
\newcommand{\vy}{{\mathbf y}}

\newcommand{\minplus}{\oplus}
\newcommand{\maxplus}{\boxplus}

\newcommand{\lowertconv}{\underline{\mathrm{tconv}}}
\newcommand{\uppertconv}{\overline{\mathrm{tconv}}}
\newcommand{\loweruppertconv}{\underline{\overline{\mathrm{tconv}}}}

\newcommand{\Span}{\mathrm{Span}}
\newcommand{\Expr}{\mathrm{Expr}}

\title{Min-Max-Plus Neural Networks}
\author{
  Ye Luo\footnote{This work is supported by National Natural Science Foundation of China (Grant No. 61875169).} \\
 School of Informatics\\
  Xiamen University\\
  Xiamen,  Fujian 361005, China \\
  \texttt{luoye@xmu.edu.cn, luoye80@gmail.com} \\
  \and
   Shiqing Fan \\
    School of Informatics\\
  Xiamen University\\
  Xiamen,  Fujian 361005, China \\
   \texttt{loy.fsq@gmail.com} \\
}

\begin{document}

\maketitle

\begin{abstract}
We present a new model of neural networks called Min-Max-Plus Neural Networks (MMP-NNs) based on operations in tropical arithmetic. In general, an MMP-NN is composed of three types of alternately stacked layers, namely  linear layers, min-plus layers and max-plus layers. Specifically, the latter two types of layers constitute the nonlinear part of the network which is trainable and more sophisticated compared to the nonlinear part of conventional neural networks. In addition, we show that with higher capability of nonlinearity expression, MMP-NNs are universal approximators of continuous functions, even when the number of multiplication operations is tremendously reduced (possibly to none in certain extreme cases). Furthermore,  we formulate the backpropagation algorithm in the training process of MMP-NNs and introduce an algorithm of normalization to improve the rate of convergence in training.
\end{abstract}

\section{Introduction}
\subsection{motivation}
Conventional artificial neural networks typically have a fixed nonlinear activation function that applies to all neurons. Introducing a trainable nonlinear part of the network usually further enhances fitting capability of the network. For example,  the seminal work of He \emph{et al.} \cite{HZRS2015} showed for the first time that human-level performance on ImageNet Classification (experimentally with an error rate of $5.1\%$) could be surpassed by the performance of a large scale deep neural network (experimentally with an error rate of $4.94\%$). A key ingredient of their work is to make an extension of the classical Rectified Linear Unit (ReLU) as the nonlinear activation function to Parametric Rectified Linear Unit (PReLU) in which the slope of the negative part of the input is learnable. 

In this paper, we propose a new framework of neural networks called Min-Max-Plus Neural Networks (MMP-NNs) whose nonlinear part is systematically complexified. The mathematical foundation of this model is called tropical mathematics \cite{MS15} which is a fast developing area in mathematics and whose connections to neural networks have been established only very recently. A special feature of tropical mathematics is that the operations of usual multiplications and additions degenerate to operations of additions (called ``tropical multiplications'') and min/max operations (called ``tropical additions'') respectively. Consequently, the nonlinear part of an MMP-NN only involves additions and min/max operations. 

Since the nonlinear part of an MMP-NN is trainable, the overall fitting capability of the model is determined by the fitting capabilities of both the linear and nonlinear parts of the network. In some cases, the network can be configured such that the majority of the fitting capability comes from the nonlinear part which only uses additions and min/max operations. As a result, the number of the multiplication operations can be tremendously reduced in the computation of such a network. In particular, using certain configurations, the number of multiplications can be even reduced to none while the network is still kept to be a universal approximator.

\subsection{Related work}
\subsubsection{Investigations of neural networks using tropical mathematics}
Our work is closely related to the applications of tropical mathematics to analyze deep neural networks. 

At present, many studies have analyzed the mechanism of deep neural networks from different perspectives, such as the advantages of deep-structured networks over shallow-structured networks \cite {DB2011,Bengio2011,Eldan2016}, the impact of activation functions on network expression capabilities \cite {Montufar2014},  explanation of  the generalization capabilities of networks \cite {ZBHRV17}, etc. The methods of analogy to neural networks in order to explain the characteristics of them, for example, using tropical geometry to simulate the structure of deep neural network, are also been proposed.

Zhang \emph{et al.} \cite {ZNL2018} proposed to use the theory of tropical geometry to analyze deep neural networks in order to explore the explanation of the characteristics of the neural network structure. In particular, they established a connection between a forward neural network with ReLU activation and tropical geometry for the first time, proving that this kind of neural networks is equivalent to the family of tropical rational mappings. They also deduced that a hidden layer of a forward ReLU neural network can be described by zonotopes as the building blocks of a deeper network, and associating the decision boundary of this neural network with a tropical hypersurface. 

In order to expand the studies of \cite {Montufar2014} about the upper bounds on linear regions of layers from ReLU activations,  leaky ReLU activations and maxout activations, \cite{CM2018} present an approach in a tropical perspective which treats neural network layers with piecewise linear activations as tropical polynomials, which are polynomials over tropical semirings.

Calafiore \emph{et al.} \cite{CGP2019} proposed a new type of neural networks called log-sum-exp (LSE) networks which are universal approximators of convex functions. In addition, they show that  difference-LSE networks are   a smooth universal approximators of continuous functions over compact convex sets \cite{CGP2020}. LSE networks and difference-LSE networks  are also closely related to tropical mathematics via the so-called ``dequantization'' procedure. 

There are also studies to use the relevant content of tropical geometry to solve the problem of optimal piecewise linear regression. Maragos \emph{et al.} \cite{MT2019} generalized tropical geometrical objects using weighted lattices and provided the optimal solution of max-$\star$ equations (max-$\star$ algebra with an arbitrary binary operation $\star$ that distributes over max) using morphological adjunctions that are projections on weighted lattices. Then by fitting max-$\star$ tropical curves and surfaces to arbitrary data that constitute polygonal or polyhedral shape approximations, the relationship between tropical geometry and optimization of piecewise-linear regression can be established. With this theory, they again proposed an approach for multivariate convex regression by using an approximation model of the maximum of hyperplanes represented as a multivariate max-plus tropical polynomial \cite{MT2020}, and show that the method has lower complexity than most other methods for fitting piecewise-linear (PWL) functions.

Newer tropical methods are also developed to formally simulate the training of some neural networks. Smyrnis \emph{et al.} \cite{SM2019} \cite{SMR2020} emulate the division of regular polynomials, when applied to those of the max-plus semiring from the aspect of tropical polynomial division. This is done via the approximation of the Newton polytope of the dividend polynomial by that of the divisor.  The process has been applied  to minimize a two-layer fully connected network for a binary classification problem, and then evaluated in various experiments to prove its ability to approximate the network with the least performance loss. Due to them encoding the totality of the information contained in the network, the results helped them to demonstrate that the Newton polytopes of the tropical polynomials corresponding to the network provide a reliable way to approximate its labels.

\subsubsection{Investigations on reducing multiplication operations in neural networks}
For a deep neural network (DNN), the computation is majorly on multiplications between floating-point weights and floating-point value activation during forward inference. It is well known that the execution of multiplication is typically slower than the execution of addition with higher energy consumption. In recent years, there are many studies on how to trade multiplication and addition to speed up the computation in deep learning.

The seminal work \cite{CBD2015} introduced BinaryConnect, a method training a DNN with binary (e.g. -1 or 1) weights during the forward and backward propagations, so that many multiply-accumulate operations can be replaced by simple accumulations. Moreover, Hubara \emph{et al.} \cite{HCSEB2016} proposed BNNs, which binarized not only weights but also activations in convolutional neural networks at runtime. In order to get greatly improved performance, Rastegari \emph{et al.} \cite{RORF2016} introduced XNOR-Networks, in which both the filters and the input to convolutional layers are binary. XNOR-Nets offer the possibility of running state-of-the-art networks on CPUs in real-time. To further increase the speed of traning binarized networks, \cite{ZWNZWZ2016} propose DoReFa-Net to train convolutional neural networks that have low bitwidth weights and activations using low bitwidth parameter gradients.

Considered from another aspect, Chen \emph{et al.} \cite {CWCSXTX2019} boldly gave up the convolution operation that involve a lot of matrix multiplication. They propose Adder Networks that maximize the use of addition while abandoning convolution operations, specifically, by taking the $L_1$-norm distance between filters and input feature as the output response instead of the results of convolution operations. And the corresponding optimization method is developed by using regularized full-precision gradients. The experimental results show that AdderNets can well approximate the performance of CNNs with the same architectures, which may have some impact on future hardware design. 

\subsection{Our contributions}

\begin{enumerate}[(i)]
\item Using tropical arithmetic, we propose a new model of neural networks called Min-Max-Plus Neural Networks which have a trainable and more sophisticated nonlinear part compared to conventional neural networks. We show that a general form of MMP-NNs is composed of three types of layers: linear layers, min-plus layers and max-plus layers, which can be represented by matrices, min-plus matrices and max-plus matrices respectively. Then we show that the computation of MMP-NNs are essentially a sequence of matrix multiplications and tropical matrix multiplications.  
\item We show that such a model of MMP-NNs is quite general in the sense that conventional maxout networks, ReLU networks, leaky/parametric ReLU networks, and the dequantization of Log-Sum-Exp (LSE) networsk can all be considered as specializations of a special type (called Type I) of MMP-NNs,  while on the other hand, MMP-NNs can be quite non-conventional. A special type (called Type II) of MMP-NNs have only one linear layer at the input end with the remaining layers being the min-plus layers and max-plus layers which are stacked together alternately. Another special type (Type III) of MMP-NNs has a structure of Type II networks attached with an additional output linear layer. With a more sophisticated nonlinear expressor in a Type III network, it is expected to have a more enhanced fitting ability than similarly sized conventional networks. 
\item We show that MMP-NNs of all types are universal approximators. In particular, we show that the space of functions expressible by the Type II networks with a fixed linear layer can be elegantly characterized using tropical convexity. Moreover, the proof that we give to Type II networks being universal approximators is quite distinct from the regular difference-of-convex-functions approach. 
\item We show that by using Type II networks, since there is only one linear layer, the number of multiplications in the computation can be tremendously reduced. This gives an advantage of using Type II networks in scenarios where computing resources are limited. 
\item We show that MMP-NNs can also be trained using backpropagation as in the conventional feedforward networks. We provide formulas for gradient calculations for the non-conventional min-plus layers and max-plus layers. 
\item We propose a normalization process to adjust the parameters in the min-plus and max-plus layers which helps to expedite the rate of convergence in training MMP-NNs. This normalization process is a generalization of the Legendre-Fenchel transformation widely used in physics and convex optimization. 
\end{enumerate}

\subsection{Organization of the paper}
The remaining of the paper is organized as follows: In Section~\ref{S:EleTrop}, we give a brief overview of the terminologies in tropical mathematics related to this work; In Section~\ref{S:Description}, we give a general description of  the building blocks and architecture of MMP-NNs,  and discuss several special types (Type I, II, and III) of MMP-NNs; In Section~\ref{S:UniApprox}, we prove that all types of MMP-NNs are universal approximators; In Section~\ref{S:train}, we describe a training method of MMP-NNs using backpropagation and introduce the normalization algorithm. 

\section{Elements of Tropical Arithmetic} \label{S:EleTrop}
In this section, we provide some preliminaries of tropical mathematics that are related to our work in this paper. One may refer to \cite{MS15,Butkovivc10} for a more comprehensive introduction to tropical mathematics. 

\subsection{Tropical operations} \label{SS:TropOp}
\begin{definition}
Let $\mbbR_{\min} := \mbbR\bigcup \{\infty\}$, $\mbbR_{\max}:=\mbbR\bigcup\{-\infty\}$, $\mbbRminmax:=\mbbR\bigcup \{\infty,-\infty\}$. 
\begin{enumerate}[(i)]
\item The \emph{min-plus algebra} $(\mbbR_{\min},\odot,\minplus)$ and the \emph{max-plus} algebra $(\mbbR_{\max},\odot,\maxplus)$ are semirings where $a\odot b:=a+b$, $a\minplus b:=\min(a,b)$ and $a\maxplus b:=\max(a,b)$. 
\item The operations $\odot$, $\minplus$ and $\maxplus$ are called \emph{tropical multiplication}, \emph{tropical lower addition} and \emph{tropical upper addition} respectively. 
\item By convention, the \emph{tropical semiring} is defined as either the min-plus algebra $(\mbbR_{\min},\odot,\minplus)$ or the max-plus algebra $(\mbbR_{\max},\odot,\minplus)$.
\end{enumerate}
\end{definition}

Note that $a\odot 0 = a$, $b\minplus \infty = b$ and $c\maxplus (-\infty) = c$ for all $a\in \mbbRminmax$, $b\in \mbbR_{\min}$ and $c\in \mbbR_{\max}$. This means that $0$ is the \emph{tropical multiplicative identity} for both min-plus and max-plus algebras, $\infty$ is the \emph{identity for tropical lower addition} and $-\infty$ is the \emph{identity of tropical upper addition}.  Moreover,  the \emph{tropical division} of $a,b\in \mbbR$ is defined as $a \oslash b := a-b$ and the \emph{tropical inverse} of $a\in \mbbRminmax$ is the negation $-a = 0\oslash a$. Note that the min-plus algebra and the max-plus algebra are isomorphic under negation. 

The min-plus  algebra  (resp. max-plus algebra) is an idempotent semiring, since $a\minplus a  = a$ (resp. $a\maxplus a  = a$) for all $a\in\mbbRminmax$. It can be easily verified that the tropical operations on $(\mbbR_{\min},\odot,\minplus)$ and $(\mbbR_{\max},\odot,\maxplus)$  satisfy the usual principles of commutativity, associativity and distributivity.

\begin{remark} \label{R:minmaxplus}
While only one of the min-plus algebra and the max-plus algebra is considered in most other works related to the tropical semiring, we need to deal with both in this work. In particular, the following property will be employed: the two addition operations $\minplus$ and $\maxplus$ also mutually satisfy the distributive law, i.e.,  $a\maxplus(b\minplus c)=(a\maxplus b)\minplus (a\maxplus c)$ and $a\minplus(b\maxplus c)=(a\minplus b)\maxplus (a\minplus c)$ for all $a,b,c\in \mbbRminmax$. 
\end{remark}

Tropical operations can be also be applied to functions.  For a topological space $X$,  let $C(X)$ be the space of real-valued continuous functions on $X$. 
Correspondingly we can define tropical operations on $C(X)$ as follows:
\begin{enumerate}[(i)]
\item For $f,g\in C(X)$, the \emph{lower tropical addition} and \emph{upper tropical addition} are defined as $f\minplus g :=\min(f,g)$ and $f\maxplus g :=\max(f,g)$ respectively where the minimum and the maximum are taken in a point-wise manner. 
\item For $f,g\in C(X)$ and $c\in \mbbR$, the \emph{tropical scalar multiplication}, \emph{tropical multiplication}, \emph{tropical division} and \emph{tropical inverse} are defined  respectively as $c\odot f  :=c+f$, $f\otimes g:=f+g$,  $f\oslash g:=f-g$ and $-f=0\oslash f$. 
\item By abuse of notation, we also let $\infty$ and $-\infty$ denote constant functions on $X$ taking values $\infty$ and $-\infty$ respectively. Then for  $f\in C(X)$, we have $f\minplus \infty = f$, $f\maxplus (-\infty) = f$, $f\minplus (-\infty) = -\infty$,  $f\maxplus \infty = \infty$, $f\otimes \infty = \infty$ and $f\otimes (-\infty) = -\infty$. 
\item $(C(X)\bigcup\{\infty\},\minplus, \otimes)$ and $(C(X)\bigcup\{-\infty\},\maxplus, \otimes)$ are idempotent semirings with operations satisfying the usual principles of commutativity, associativity, distributivity, and the property $f\maxplus(f\minplus h)=(f\maxplus g)\minplus (f\maxplus h)$ and $f\minplus(g\maxplus h)=(f\minplus g)\maxplus (f\minplus h)$ for all $f,g,h \in C(X)\bigcup\{\infty,-\infty\}$. 
\end{enumerate}

\subsection{Tropical convexity} \label{SS:tconv} 

Now let us introduce the notion of tropical convexity \cite{DS04, Luo2018} for both min-plus algebra and max-plus algebra. 

\begin{definition}
A subset $T$ of $C(X)$ is said to be \emph{lower tropically convex} (respectively \emph{upper tropically convex}) if $a\odot f \minplus b\odot g$ (respectively $a\odot f \maxplus b\odot g$) is contained in $T$ for all $a,b\in \mbbR$ and all $f, g \in T$. 
\end{definition}

Let $f_1,\cdots,f_n\in C(X)$, $a_1,\cdots,a_n \in \mbbR_{\min}$ and $b_1,\cdots,b_n \in \mbbR_{\max}$. Then we say $a_1\odot f_1\minplus \cdots \minplus a_n\odot f_n$ is a \emph{min-plus combination} of $f_1,\cdots,f_n$ and $b_1\odot f_1\maxplus \cdots \minplus b_n\odot f_n$ is a \emph{max-plus combination} of $f_1,\cdots,f_n$. Here we also call both min-plus and max-plus combinations \emph{ tropical linear combinations}. 

\begin{definition}
Let $V$ be a subset of $C(X)$. The \emph{lower tropical convex hull} (respectively \emph{upper tropical convex hull}) of $V$ is the smallest lower (respectively upper) tropically convex subset of $C(X)$ containing $V$.
\end{definition}

We may characterize the tropical convex hulls as sets of tropical linear combinations as stated in the following proposition. (For more details, see Section~3 of \cite{Luo2018}). 

\begin{proposition} \label{P:TropConv}
For a subset $V$ of $C(X)$, the lower (respectively upper) tropical convex hull of $V$ is composed of all min-plus (respectively max-plus) linear combinations of elements in $V$, i.e., $\lowertconv(V)  = \{(a_1\odot f_1)\minplus\cdots\minplus (a_m\odot f_m)\mid m\in\mbbN, a_i\in \mbbR,f_i\in V\}$ and $\uppertconv(V)  = \{(b_1\odot f_1)\maxplus\cdots\maxplus (b_m\odot f_m)\mid m\in\mbbN, b_i\in \mbbR,f_i\in V\}$. 
\end{proposition}

\subsection{Tropical polynomials and posynomials}
A \emph{polynomial} in $n$ variables with coefficients in $\mbbR$ has an expression of the form $p(\vx) = \sum\limits_{i=1}^{m}c_ix_1^{a_{i1}}\cdots x_n^{a_{in}}$ with $\vx = (x_1,\cdots,x_n)$, $c_i\in\mbbR$ and $a_{ij}\in \mbbN$. The terms $c_ix_1^{a_{i1}}\cdots x_n^{a_{in}}$ are called \emph{monomials}. Polynomials and monomials can be considered as functions on $\mbbR^n$. If in addition the coefficients $c_i$ and the coordinates $x_i$ are required to be positive reals while the assumption on the exponents  is relaxed such that $a_{ij}$ are allowed to be any real numbers, then $p(\vx)$ is a called a  \emph{posynomial}.

\begin{definition}
Using tropical operations, the  tropical counterparts of monomials, polynomials, and posynomials are defined as follows:
\begin{enumerate}[(i)]
\item A \emph{tropical monomial} in $n$ variables is a function of $\vx=(x_1,\cdots,x_n)\in \mbbR^n$ of the form $c+\langle\va,\vx\rangle$ with $c\in \mbbR$ and $\va=(a_1,\cdots,a_n)\in\mbbR^n$. Here $\langle\va,\vx\rangle=a_1x_1+\cdots+a_nx_n$ is the inner product of the vectors $\va$ and $\vx$. 
\item A \emph{min-plus polynomial} (respectively a \emph{max-plus polynomial}) in $n$ variables is a function on $\mbbR^n$ of the form $p(\vx) =\sum_{i=1,\cdots,m}^{\minplus}c_i\odot x_1^{\otimes a_{i1}}\cdots x_n^{\otimes a_{in}} =\min_{i=1,\cdots,m}\left(c_i+\langle\va_i,\vx\rangle\right)$ (respectively $p(\vx) =\sum_{i=1,\cdots,m}^{\maxplus}c_i\odot x_1^{\otimes a_{i1}}\cdots x_n^{\otimes a_{in}} =\max_{i=1,\cdots,m}\left(c_i+\langle\va_i,\vx\rangle\right)$) where $c_i\in\mbbR$, $\vx = (x_1,\cdots,x_n) \in \mbbR^n$, $\va_i = (a_{i1},\cdots,a_{in}) \in \mbbN^n$ and $\langle\va_i,\vx\rangle=a_{i1}x_1+\cdots +a_{in}x_n$ is the inner product of the vectors $\va_i$ and $\vx$. 
\item If instead of requiring $\va\in\mbbN^n$ we allow $\va\in\mbbR^n$, then the function $\min_{i=1,\cdots,m}\left(c_i+\langle\va_i,\vx\rangle\right)$ is called a \emph{min-plus posynomial} and the function $\max_{i=1,\cdots,m}\left(c_i+\langle\va_i,\vx\rangle\right)$ is called a \emph{max-plus posynomial}. 

\end{enumerate}
\end{definition}

\begin{remark}
In the context of this paper, the coefficient vector $\va=(a_1,\cdots,,a_n)$ of a tropical monomial $c+\langle\va,\vx\rangle$ is not necessarily restricted to be a vector of integers  (as in many other works) and our discussions will be mainly focused on tropical posynomials rather than tropical polynomials. By the above definitions, tropical monomials are simply affine functions, min-plus (respectively max-plus) posynomials are precisely convex-upward (respectively convex-downward) piecewise-linear functions. 
\end{remark}

\subsection{Tropical matrix algebra}
Using tropical operations, we can also define additions and products of tropical matrices. 
A \emph{min-plus matrix} is a matrix $(a_{ij})_{ij}$ with entries $a_{ij}\in \mbbR_{\min}$ and a \emph{max-plus matrix} is a matrix $(a_{ij})_{ij}$ with entries $a_{ij}\in \mbbR_{\max}$. Then we may denote the space of $m\times n$ min-plus matrices by $\mbbR_{\min}^{m\times n}$ and the space of $m\times n$ max-plus matrices by $\mbbR_{\max}^{m\times n}$.

\begin{definition}
Consider min-plus matrices $A = (a_{ij})\in \mbbR_{\min}^{m\times n}$, $B = (b_{ij})\in \mbbR_{\min}^{m\times n}$ and $C = (c_{ij})\in \mbbR_{\min}^{n\times p}$, and max-plus matrices $A' = (a'_{ij})\in \mbbR_{\max}^{m\times n}$, $B' = (b'_{ij})\in \mbbR_{\min}^{m\times n}$ and $C' = (c'_{ij})\in \mbbR_{\min}^{n\times p}$. 
\begin{enumerate}[(i)]
\item The \emph{min-plus sum} $S = A \minplus B$ is defined to be the $m\times n$ min-plus matrix $S=(s_{ij})_{ij} $ with entries $s_{ij}= a_{ij} \minplus b_{ij}=\min(a_{ij},b_{ij})$ for $i=1,\cdots,m$ and $j=1,\cdots, n$. 
\item The \emph{max-plus sum} $S' = A' \maxplus B'$ is defined to be the $m\times n$ max-plus matrix $S'=(s'_{ij})_{ij} $ with entries $s'_{ij}= a'_{ij} \maxplus b'_{ij}=\max(a'_{ij},b'_{ij})$ for $i=1,\cdots,m$ and $j=1,\cdots, n$. 
\item The \emph{min-plus matrix product} $T=A\underline{\otimes}C$ is defined to be the $m\times p$ min-plus matrix $T=(t_{ij})_{ij} $ with entries
$t_{ij} = a_{i1}\odot b_{1j} \minplus \cdots \minplus a_{in}\odot b_{nj}=\min_{k=1,\cdots,n}(a_{ik}+b_{kj})$ for $i=1,\cdots,m$ and $j=1,\cdots, p$. 

\item The  \emph{max-plus matrix product} $T'=A'\overline{\otimes}C'$ is defined to be the $m\times p$ max-plus matrix $T'=(t'_{ij})_{ij} $ with entries
$t'_{ij} = a'_{i1}\odot b'_{1j} \maxplus \cdots \maxplus a'_{in}\odot b'_{nj}=\max_{k=1,\cdots,n}(a'_{ik}+b'_{kj})$ for $i=1,\cdots,m$ and $j=1,\cdots, p$. 
\item Let $\underline{I}_n = \left(\begin{array}{ccccc}0 & \infty & \infty & \cdots & \infty \\  \infty  &0 & \infty  & \cdots & \infty \\ \infty & \infty &0  & \cdots & \infty \\ \vdots & \vdots &\vdots &\ddots &\vdots  \\ \infty & \infty & \infty & \cdots & 0 \end{array}\right)  \in \mbbR_{\min}^{n\times n}$ and  $\overline{I}_n = \left(\begin{array}{ccccc}0 & -\infty & -\infty & \cdots & -\infty \\  -\infty  &0 & -\infty  & \cdots & -\infty \\ -\infty & -\infty &0  & \cdots & -\infty \\ \vdots & \vdots &\vdots &\ddots &\vdots  \\ -\infty & -\infty & -\infty & \cdots & 0 \end{array}\right)  \in \mbbR_{\max}^{n\times n}$. Then $\underline{I}_n$ is called the $n\times n$ \emph{min-plus identity matrix} and $\overline{I}_n$ is called the $n\times n$ \emph{max-plus identity matrix}, since $\underline{I}_n\underline{\otimes} A = A \underline{\otimes}  \underline{I}_n=A$ for all $A\in \mbbR_{\min}^{n\times n}$ and $\overline{I}_n\overline{\otimes} B = B \overline{\otimes}  \overline{I}_n = B$ for all $B\in \mbbR_{\max}^{n\times n}$. 
\end{enumerate}
\end{definition}

\begin{example}
Let $A = \left(\begin{array}{cc}7 & 2 \\ 0  & -1 \\ 3 & 4\end{array}\right)$, $B =\left(\begin{array}{cc}5 & 3 \\ 6  & 2 \end{array}\right)$. Then it can be easily verified that $A\underline{\otimes}B = \left(\begin{array}{cc}8 & 4 \\ 5  & 1 \\ 8 & 6\end{array}\right)$ and $A\overline{\otimes}B = \left(\begin{array}{cc}12 & 10 \\ 5  & 3 \\ 10 & 6\end{array}\right)$.
\end{example}

\begin{remark} \label{R:Min-Min-Transform}
Let $A \in \mbbR_{\min}^{m\times n}$, $A' \in \mbbR_{\max}^{m\times n}$ and $B \in \mbbR^{n\times p}$. Suppose each row vector in $A$ contains at least one  entry  not being $\infty$ and each row vector in $A'$ contains at least one entry not being $-\infty$.  Then  $A\underline{\otimes} B\in \mbbR^{m\times p}$ and $A'\overline{\otimes} B\in \mbbR^{m\times p}$. By letting $p=1$, this actually means both $A$ and $A'$ can be considered as ``tropical linear'' transformations (called \emph{min-plus transformation} and \emph{max-plus transformation} respectively) from $\mbbR^n$ to $\mbbR^m$.
\end{remark}

\section{Architecture of Min-Max-Plus Networks} \label{S:Description}
In this section, we give a description of  the building blocks and general form of MMP-NNs,  and discuss in addition several special types (Type I, II, and III) of MMP-NNs.

\subsection{Building blocks: linear layers, min-plus layers and max-plus layers}

In general, an Min-Max-Plus Neural Network (MMP-NN) is composed of three types of layers: linear layers, min-plus layers and max-plus layers. As shown in Figure~\ref{F:Linear-Min-Max}, a linear layer is represented by a linear transformation, a min-plus layer is represented by a min-plus transformation and a max-plus layer is represented by a max-plus  transformation (Remark~\ref{R:Min-Min-Transform}). 

More precisely, suppose the input of a linear layer, a min-plus player or a max-plus layer is an $n$-dimensional vector in $\mbbR^n$ and the output is an $m$-dimensional vector  in $\mbbR^m$. Then the corresponding linear transformation $\lambda: \mbbR^n\to \mbbR^m$, min-plus transformation $\alpha: \mbbR^n\to \mbbR^m$ and max-plus transformation $\beta: \mbbR^n\to \mbbR^m$ can be expressed as $\rho(\vx) = L\cdot \vx$, $\alpha(\vx) = A\underline{\otimes} \vx$ and $\beta(\vx) = B\overline{\otimes} \vx$ respectively where $\vx\in \mbbR^n$, $L\in \mbbR^{m\times n}$, $A\in\mbbR_{\min}^{m\times n}$ and $B\in \mbbR_{\max}^{m\times n}$. Since the output should be in $\mbbR^m$, we require that each row vector in $A$ contains at least one entry not being $\infty$ and each row vector in $B$ contains at least one entry not being $-\infty$ (Remark~\ref{R:Min-Min-Transform}). Figure~\ref{F:Linear-Min-Max} shows an example of a linear layer, a min-layer and a max-layer, all for $m=3$ and $n=2$. 

\begin{remark}
It should be emphasized that even though min-plus transformations and max-plus transformations are ``tropically linear'', in general they are nonlinear in the usual sense, i.e., piecewise linear but convex upward or downward. 
\end{remark}

\begin{figure}[tbp] 
\centering

\begin{tikzpicture}[>=to,x=1cm,y=0.5cm, scale=1.5]

\node[draw, circle] (x1) at (0,8) {$x_1$};
\node[draw, circle] (x2) at (0,6) {$x_2$};
\node[draw, circle] (y1) at (1.8,9) {$y_1$};
\node[draw, circle] (y2) at (1.8,7) {$y_2$};
\node[draw, circle] (y3) at (1.8,5) {$y_3$};

\begin{scope}[line width=1.6pt, every node/.style={sloped,allow upside down}]
  \draw (0.9,10) node[anchor=south] {(a) Linear Layer};
  \draw[->,line width=1pt] (x1) edge (y1) (x1) edge (y2) (x1) edge (y3);
  \draw[->,line width=1pt] (x2) edge (y1) (x2) edge (y2) (x2) edge (y3);
  \draw (0.9,3) node[font=\large, align=center, scale=0.8] {$\left(\begin{array}{l}y_{1} \\ y_{2} \\ y_{3}\end{array}\right)=\left(\begin{array}{ll}w_{11} & w_{12} \\ w_{21} & w_{22} \\ w_{31} & w_{32}\end{array}\right)\cdot\left(\begin{array}{l}x_{1} \\ x_{2}\end{array}\right)$};
\end{scope}

\node[draw, circle] (x3) at (4,8) {$x'_1$};
\node[draw, circle] (x4) at (4,6) {$x'_2$};
\node[draw, circle] (y4) at (5.8,9) {$y'_1$};
\node[draw, circle] (y5) at (5.8,7) {$y'_2$};
\node[draw, circle] (y6) at (5.8,5) {$y'_3$};

\begin{scope}[line width=1.6pt, every node/.style={sloped,allow upside down}]
  \draw (4.9,10) node[anchor=south] {(b) Min-Plus Layer};
  \draw[->,line width=1pt] (x3) edge (y4) (x3) edge (y5) (x3) edge (y6);
  \draw[->,line width=1pt] (x4) edge (y4) (x4) edge (y5) (x4) edge (y6);
  \draw (4.9,3) node[font=\large, align=center, scale=0.8] {$\left(\begin{array}{l}y_{1}' \\ y_{2}' \\ y_{3}'\end{array}\right)=\left(\begin{array}{ll}w_{11}' & w_{12}' \\ w_{21}' & w_{22}' \\ w_{31}' & w_{32}'\end{array}\right) \underline{\otimes} \left(\begin{array}{l}x_{1}' \\ x_{2}'\end{array}\right)$};
\end{scope}

\node[draw, circle] (x5) at (8,8) {$x''_1$};
\node[draw, circle] (x6) at (8,6) {$x''_2$};
\node[draw, circle] (y7) at (9.8,9) {$y''_1$};
\node[draw, circle] (y8) at (9.8,7) {$y''_2$};
\node[draw, circle] (y9) at (9.8,5) {$y''_3$};

\begin{scope}[line width=1.6pt, every node/.style={sloped,allow upside down}]
  \draw (8.9,10) node[anchor=south] {(c) Max-Plus Layer};
  \draw[->,line width=1pt] (x5) edge (y7) (x5) edge (y8) (x5) edge (y9);
  \draw[->,line width=1pt] (x6) edge (y7) (x6) edge (y8) (x6) edge (y9);
  \draw (8.9,3) node[font=\large, align=center, scale=0.8] {$\left(\begin{array}{l}y_{1}'' \\ y_{2}'' \\ y_{3}''\end{array}\right)=\left(\begin{array}{ll}w_{11}'' & w_{12}'' \\ w_{21}'' & w_{22}'' \\ w_{31}'' & w_{32}''\end{array}\right) \overline{\otimes} \left(\begin{array}{l}x_{1}'' \\ x_{2}''\end{array}\right)$};
\end{scope}
\end{tikzpicture}

\caption{(a) A linear layer is representd by a linear transformation. (b) A min-plus layer is represented by a min-plus linear transformation. (c) A max-plus layer is represented by a max-plus linear transformation.} \label{F:Linear-Min-Max}
\end{figure}
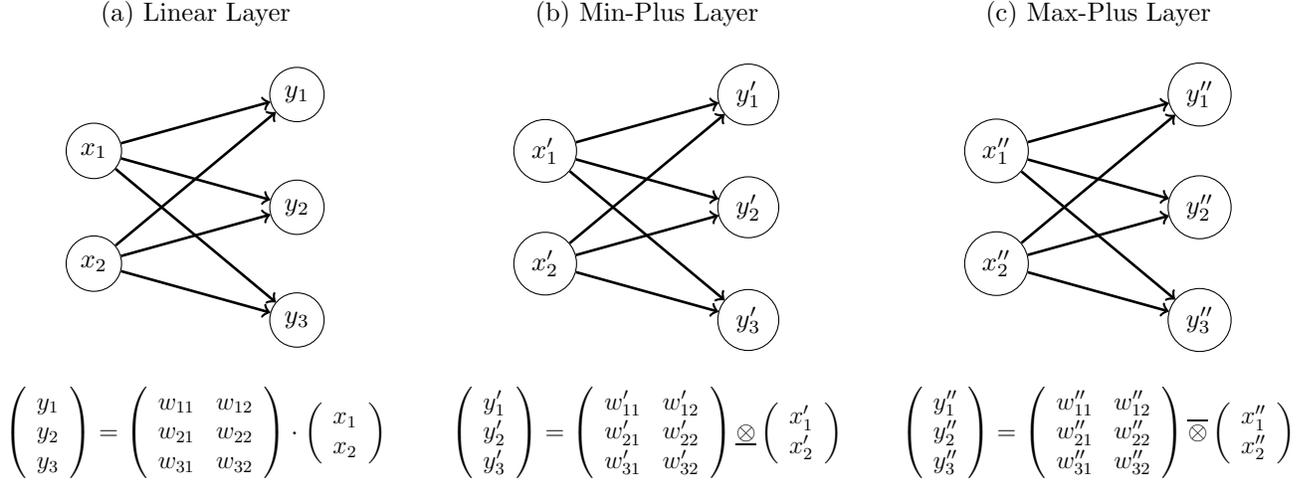  

\subsection{The general form of MMP-NNs}
\begin{figure}[tbp]
\begin{tikzpicture}[>=to,x=1cm,y=0.5cm, scale=1.4]

\begin{scope}[line width=1.6pt, every node/.style={sloped,allow upside down}, yshift=0cm]
  \draw (0,9.7) node[anchor=west] {(a) General form};

  \draw[-latex,line width=3pt] (0,8)--(0.5,8);
  \draw (0.25,8) node[above=0.8] { Input};
  
  \filldraw[fill=yellow,line width=0pt,fill opacity=0.5](1,9)--(2,9)--(2,7)--(1,7);
  \filldraw[fill=red,line width=0pt,fill opacity=0.5](2,9)--(3,9)--(3,7)--(2,7);
  \filldraw[fill=blue,line width=0pt,fill opacity=0.5](3,9)--(4,9)--(4,7)--(3,7);

  \filldraw[fill=yellow,line width=0pt,fill opacity=0.5](5,9)--(6,9)--(6,7)--(5,7);
  \filldraw[fill=red,line width=0pt,fill opacity=0.5](6,9)--(7,9)--(7,7)--(6,7);
  \filldraw[fill=blue,line width=0pt,fill opacity=0.5](7,9)--(8,9)--(8,7)--(7,7);
  
  \draw[-,line width=1.5pt] (1,9) -- (1,7);
  \draw[-,line width=1.5pt] (4,9) -- (4,7);
  \draw[-,line width=1.5pt] (1,9) -- (4,9);
  \draw[-,line width=1.5pt] (1,7) -- (4,7);
  
  \draw[-,line width=0.5pt] (2,9) -- (2,7);
  \draw[-,line width=0.5pt] (3,9) -- (3,7);
  
  \draw (1.5,8) node[scale=0.8] {Linear};
  \draw (2.5,8) node[scale=0.8] {Min-Plus};
  \draw (3.5,8) node[scale=0.8] {Max-Plus};
  
  \draw (4.5,8) node[scale=1.5] {$\cdots$};
  
  \draw[-,line width=1.5pt] (5,9) -- (5,7);
  \draw[-,line width=1.5pt] (8,9) -- (8,7);
  \draw[-,line width=1.5pt] (5,9) -- (8,9);
  \draw[-,line width=1.5pt] (5,7) -- (8,7);
  
  \draw[-,line width=0.5pt] (6,9) -- (6,7);
  \draw[-,line width=0.5pt] (7,9) -- (7,7);
  
  \draw (5.5,8) node[scale=0.8] { Linear};
  \draw (6.5,8) node[scale=0.8] { Min-Plus};
  \draw (7.5,8) node[scale=0.8] { Max-Plus};

  \draw[-latex,line width=3pt] (8.5,8)--(9,8);
  \draw (8.75,8) node[above=0.8] { Output};
\end{scope}


\begin{scope}[line width=1.6pt, every node/.style={sloped,allow upside down}, yshift=-2cm]
 \draw (0,9.7) node[anchor=west] {(b) Type \uppercase\expandafter{\romannumeral1}};
  
  \draw[-latex,line width=3pt] (0,8)--(0.5,8);
  \draw (0.25,8) node[above=0.8] { Input};
  
  \filldraw[fill=yellow,line width=0pt,fill opacity=0.5](1,9)--(2,9)--(2,7)--(1,7);
  \filldraw[fill=blue,line width=0pt,fill opacity=0.5](2,9)--(3,9)--(3,7)--(2,7);
  \filldraw[fill=yellow,line width=0pt,fill opacity=0.5](3,9)--(4,9)--(4,7)--(3,7);
  \filldraw[fill=blue,line width=0pt,fill opacity=0.5](4,9)--(5,9)--(5,7)--(4,7);

  \filldraw[fill=yellow,line width=0pt,fill opacity=0.5](6,9)--(7,9)--(7,7)--(6,7);
  \filldraw[fill=blue,line width=0pt,fill opacity=0.5](7,9)--(8,9)--(8,7)--(7,7);
  
  \draw[-,line width=1.5pt] (1,9) -- (1,7);
  \draw[-,line width=1.5pt] (5,9) -- (5,7);
  \draw[-,line width=1.5pt] (1,9) -- (5,9);
  \draw[-,line width=1.5pt] (1,7) -- (5,7);
  
  \draw[-,line width=0.5pt] (2,9) -- (2,7);
  \draw[-,line width=1.5pt] (3,9) -- (3,7);
  \draw[-,line width=0.5pt] (4,9) -- (4,7);
  
  \draw (1.5,8) node[scale=0.8] {Linear};
  \draw (2.5,8) node[scale=0.8] { Max-Plus};
  \draw (3.5,8) node[scale=0.8] {Linear};
  \draw (4.5,8) node[scale=0.8] { Max-Plus};
  
  \draw (5.5,8) node[scale=1.5] {$\cdots$};
  
  \draw[-,line width=1.5pt] (6,9) -- (6,7);
  \draw[-,line width=1.5pt] (8,9) -- (8,7);
  \draw[-,line width=1.5pt] (6,9) -- (8,9);
  \draw[-,line width=1.5pt] (6,7) -- (8,7);
  
  \draw[-,line width=0.5pt] (7,9) -- (7,7);
  
  \draw (6.5,8) node[scale=0.8] { Linear};
  \draw (7.5,8) node[scale=0.8] { Max-Plus};

  \draw[-latex,line width=3pt] (8.5,8)--(9,8);
  \draw (8.75,8) node[above=0.8] { Output};
\end{scope}


\begin{scope}[line width=1.6pt, every node/.style={sloped,allow upside down}, yshift=-4cm]
  \draw (0,9.7) node[anchor=west] {(c) Type \uppercase\expandafter{\romannumeral2}};
  
  \draw[-latex,line width=3pt] (0,8)--(0.5,8);
  \draw (0.25,8) node[above=0.8] { Input};
  
  \filldraw[fill=yellow,line width=0pt,fill opacity=0.5](1,9)--(2,9)--(2,7)--(1,7);
  \filldraw[fill=red,line width=0pt,fill opacity=0.5](2,9)--(3,9)--(3,7)--(2,7);
  \filldraw[fill=blue,line width=0pt,fill opacity=0.5](3,9)--(4,9)--(4,7)--(3,7);
  \filldraw[fill=red,line width=0pt,fill opacity=0.5](4,9)--(5,9)--(5,7)--(4,7);
  \filldraw[fill=blue,line width=0pt,fill opacity=0.5](5,9)--(6,9)--(6,7)--(5,7);

  \filldraw[fill=red,line width=0pt,fill opacity=0.5](7,9)--(8,9)--(8,7)--(7,7);
  \filldraw[fill=blue,line width=0pt,fill opacity=0.5](8,9)--(9,9)--(9,7)--(8,7);
  
  \draw[-,line width=1.5pt] (1,9) -- (1,7);
  \draw[-,line width=1.5pt] (6,9) -- (6,7);
  \draw[-,line width=1.5pt] (1,9) -- (6,9);
  \draw[-,line width=1.5pt] (1,7) -- (6,7);
  
  \draw[-,line width=1.5pt] (2,9) -- (2,7);
  \draw[-,line width=0.5pt] (3,9) -- (3,7);
  \draw[-,line width=1.5pt] (4,9) -- (4,7);
  \draw[-,line width=0.5pt] (5,9) -- (5,7);
  
  \draw (1.5,8) node[scale=0.8] {Linear};
  \draw (2.5,8) node[scale=0.8] { Min-Plus};
  \draw (3.5,8) node[scale=0.8] {Max-Plus};
  \draw (4.5,8) node[scale=0.8] { Min-Plus};
  \draw (5.5,8) node[scale=0.8] { Max-Plus};
  
  \draw (6.5,8) node[scale=1.5] {$\cdots$};
  
  \draw[-,line width=1.5pt] (7,9) -- (7,7);
  \draw[-,line width=1.5pt] (9,9) -- (9,7);
  \draw[-,line width=1.5pt] (7,9) -- (9,9);
  \draw[-,line width=1.5pt] (7,7) -- (9,7);
  
  \draw[-,line width=0.5pt] (7,9) -- (7,7);
  
  \draw (7.5,8) node[scale=0.8] { Min-Plus};
  \draw (8.5,8) node[scale=0.8] { Max-Plus};

  \draw[-latex,line width=3pt] (9.5,8)--(10,8);
  \draw (9.75,8) node[above=0.8] { Output};
\end{scope}


\begin{scope}[line width=1.6pt, every node/.style={sloped,allow upside down}, yshift=-6cm]
  \draw (0,9.7) node[anchor=west] {(d) Type \uppercase\expandafter{\romannumeral3}};
  
  \draw[-latex,line width=3pt] (0,8)--(0.5,8);
  \draw (0.25,8) node[above=0.8] { Input};
  
  \filldraw[fill=yellow,line width=0pt,fill opacity=0.5](1,9)--(2,9)--(2,7)--(1,7);
  \filldraw[fill=red,line width=0pt,fill opacity=0.5](2,9)--(3,9)--(3,7)--(2,7);
  \filldraw[fill=blue,line width=0pt,fill opacity=0.5](3,9)--(4,9)--(4,7)--(3,7);
  \filldraw[fill=red,line width=0pt,fill opacity=0.5](4,9)--(5,9)--(5,7)--(4,7);
  \filldraw[fill=blue,line width=0pt,fill opacity=0.5](5,9)--(6,9)--(6,7)--(5,7);

  \filldraw[fill=red,line width=0pt,fill opacity=0.5](7,9)--(8,9)--(8,7)--(7,7);
  \filldraw[fill=blue,line width=0pt,fill opacity=0.5](8,9)--(9,9)--(9,7)--(8,7);
  \filldraw[fill=yellow,line width=0pt,fill opacity=0.5](9,9)--(10,9)--(10,7)--(9,7);
  
  \draw[-,line width=1.5pt] (1,9) -- (1,7);
  \draw[-,line width=1.5pt] (6,9) -- (6,7);
  \draw[-,line width=1.5pt] (1,9) -- (6,9);
  \draw[-,line width=1.5pt] (1,7) -- (6,7);
  
  \draw[-,line width=1.5pt] (2,9) -- (2,7);
  \draw[-,line width=0.5pt] (3,9) -- (3,7);
  \draw[-,line width=1.5pt] (4,9) -- (4,7);
  \draw[-,line width=0.5pt] (5,9) -- (5,7);
  
  \draw (1.5,8) node[scale=0.8] {Linear};
  \draw (2.5,8) node[scale=0.8] { Min-Plus};
  \draw (3.5,8) node[scale=0.8] {Max-Plus};
  \draw (4.5,8) node[scale=0.8] { Min-Plus};
  \draw (5.5,8) node[scale=0.8] { Max-Plus};
  
  \draw (6.5,8) node[scale=1.5] {$\cdots$};
  
  \draw[-,line width=1.5pt] (7,9) -- (7,7);
  \draw[-,line width=1.5pt] (9,9) -- (9,7);
  \draw[-,line width=1.5pt] (7,9) -- (10,9);
  \draw[-,line width=1.5pt] (7,7) -- (10,7);
  \draw[-,line width=1.5pt] (10,9) -- (10,7);
  
  \draw[-,line width=0.5pt] (7,9) -- (7,7);
  
  \draw (7.5,8) node[scale=0.8] { Min-plus};
  \draw (8.5,8) node[scale=0.8] { Max-Plus};
  \draw (9.5,8) node[scale=0.8] { Linear};

  \draw[-latex,line width=3pt] (10.5,8)--(11,8);
  \draw (10.75,8) node[above=0.8] { Output};
\end{scope}
\end{tikzpicture}

\caption{(a) The general form of MMP-NNs; (b) Type I MMP-NNs;   (c) Type II MMP-NNs; (d) Type III MMP-NNs;}  \label{F:MMP-NN}
\end{figure}
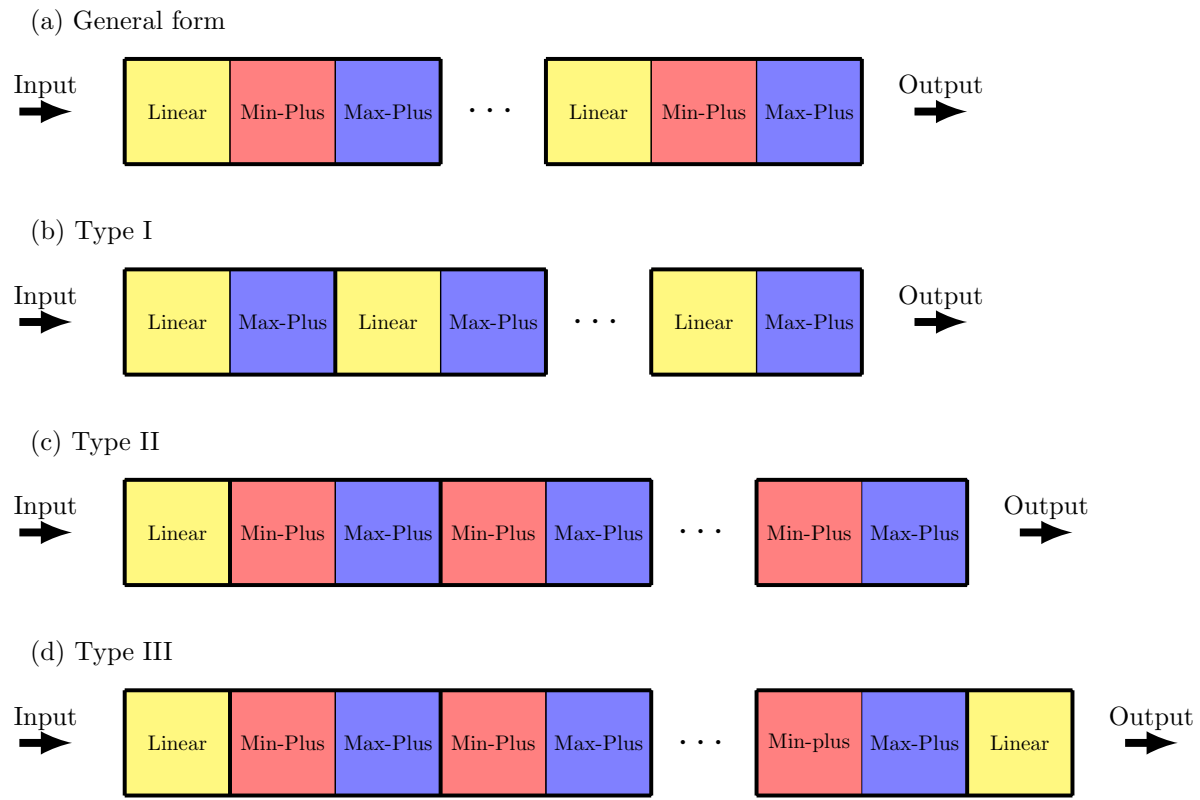

Figure~\ref{F:MMP-NN}(a) shows the general form of a feedforward multilayer MMP-NN, which is composed of a sequence of composite layers, each containing a linear layer, a min-plus layer and a max-plus layer. Consider an MMP-NN with $K$ composite layers and denote the linear transformation, min-plus transformation and max-plus transformation of the $k$-th layer by $\lambda_k$, $\alpha_k$ and $\beta_k$ respectively. 

Let  $d$ be the number of input nodes and $p$ be the number of output nodes of the MMP-NN. Then the MMP-NN is a function $\Phi: \mbbR^d\to \mbbR^p$ which can be written as a composition of linear, min-plus and max-plus transformations $$\Phi=\beta_K\circ \alpha_K\circ \lambda_K \circ \beta_{K-1} \circ \alpha_{K-1} \circ \lambda_{K-1}\circ \cdots \circ \beta_1\circ\alpha_1\circ\lambda_1.$$

Let $d_i$, $n_i$ and $m_i$ be the number of input nodes of the $k$-th linear layer, the $k$-th min-plus layer and the $k$-th max-plus layer respectively. Note that  $d_1=d$ and we let $d_{K+1} = p$. Then we have  $\lambda_k:\mbbR^{d_k}\to \mbbR^{n_k}$,  $\alpha_k:\mbbR^{n_k}\to \mbbR^{m_k}$, $\beta_k:\mbbR^{m_k}\to  \mbbR^{d_{k+1}}$ for $k=1,\cdots,K$. Let $L_k\in \mbbR^{n_k\times d_k}$, $A_k\in\mbbR_{\min}^{m_k\times n_k}$ and $B_k\in\mbbR_{\max}^{d_{k+1}\times m_k}$ be the corresponding matrix, min-plus matrix and max-plus matrix for $\lambda_k$, $\alpha_k$ and $\beta_k$ respectively. Then for each $\vx\in\mbbR^d$, we have $$\Phi(\vx) =(B_K \overline {\otimes}(A_K \underline{\otimes}(L_K \cdot (B_{K-1}\overline{\otimes} \cdots \underline{\otimes}(L_2 \cdot(B_1\overline{\otimes}(A_1\underline{\otimes}(L_1\cdot \vx)\cdots).$$

\begin{remark}
It should be emphasized that there is no associative law for the above intermediate transformations. For example, suppose $A_1$ and $B_1$ both have entries in $\mbbR$. To compute $\vy = B_1\overline{\otimes}(A_1\underline{\otimes}(L_1\cdot \vx))$, one has to compute step by step as $\vu = L_1\cdot \vx$, $\vv = A_1\underline{\otimes}\vu$ and $\vy = B_1\overline{\otimes}\vv$. If one computes $\vu = L_1\cdot \vx$ and $C=B_1\overline{\otimes}A_1$ first and then computes  $C\underline{\otimes} \vu$, the result won't agree with the above step-by-step computation in general. 
\end{remark}

Conventionally, a multilayer feedforward neural network is typically composed of a sequence of  affine transformations and nonlinear activations. The activation function is usually fixed and relatively simple, e.g. ReLU, sigmoid, etc. Even though the nonlinearity added by each neuron of activation is rather small, the totally nonlinearity of the system can be accumulated after layers of tranformations and activations. 

For a single composite layer of a general MMP-NN,  a linear layer $\lambda_k$ is followed by a min-plus layer $\alpha_k$ and  a max-plus layer $\beta_k$ , while the composition of $\alpha_k$ and $\beta_k$ can be considered as a nonlinear activation. This composition can complexify the system to high nonlinearity in one activation. In addition, we allow to train the parameters in all three types of matrices (linear, min-plus and max-plus), as will be discussed in detail in Section~\ref{S:train}. 

\subsection{MMP-NNs of special types: Type I, Type II and Type III}

The architecture of the system can be simplified by dropping out some layers (or simply let the corresponding dropout layer be represented by the identity linear, min-plus or max-plus matrix). Here we show two examples of such simplifications: Figure~\ref{F:MMP-NN}(b) shows an MMP-NN (called Type I) with the min-plus layer removed for each composite layer; Figure~\ref{F:MMP-NN}(c) shows an MMP-NN (called Type II) with the linear layer removed for each composite layer except the first one; Figure~\ref{F:MMP-NN}(d) shows an MMP-NN (called Type III) which is a Type II network attached with an additional linear layer to the output end.

While the structures of Type II and Type III networks are non-conventional which we will discuss more in Section~\ref{S:UniApprox}, Type I networks actually have connections to several conventional networks as explained in the following:

\begin{enumerate}[(i)]
\item \textbf{Maxout activation.} Type I networks are essentially equivalent to maxout networks. 

Recall that in a maxout unit \cite{GWMCB2013}, for an input $\vx\in \mbbR^d$, given $W^i = (W^i_j) \in \mbbR^{n_i\times d}$ where $W^i_j$ is the $j$-th row of $W^i$ and $\vb^i=(b^i_j)\in\mbbR^{n_i}$ for $i=1,\cdots,m$, the output is $\vy=(y_i)\in \mbbR^m$ where $y_i = \max_{j=1,\cdots, m}(W^i_j\vx+b^i_j)$. Here $W^i$ and $\vb^i$ are learned parameters. 

Let $n=\sum_{i=1}^m n_i$.  A maxout network can be translated to a Type I network with a linear layer represented by matrix $L$ followed by a max-plus layer represented by a max-plus matrix $B$. More precisely, we have $\vy = B\overline{\otimes} (L\cdot \vx)$ with $L = \left(\begin{array}{c}W^1 \\ \vdots \\ W^m\end{array}\right)\in \mbbR^{n\times d}$ and $$B=\left(\begin{array}{ccccccccc} b^1_1  & \cdots & b^1_{n_1} & -\infty & -\infty &&  \cdots &  & -\infty \\ -\infty & \cdots & -\infty & b^2_1  & \cdots & b^2_{n_2} & -\infty & \cdots & -\infty  \\ &\vdots &&& \vdots &&&\vdots &  \\  -\infty & -\infty &&  \cdots &  & -\infty & b^m_1  & \cdots & b^m_{n_m}  \end{array}\right)\in\mbbR_{\max}^{m\times n}.$$

\item \textbf{ReLU activation.} For an input $\vx\in \mbbR^d$, consider an affine transformation followed by ReLU units. Given $W = (W_i)\in \mbbR^{m\times d}$ where $W_i$ is the $i$-th row of $W$ and $\vb=(b_i)\in\mbbR^m$, the output is $\vy = (y_i)\in \mbbR^m$ which is computed by $y_i = \max(W_i\cdot \vx+b_i,0)$.  

The above computation can be translated to a Type I network computation. Let $L= \left(\begin{array}{c}W \\ 0 \end{array}\right)\in \mbbR^{(m+1)\times d}$ and $B=\left(\begin{array}{cccccc}b_1 & -\infty & \cdots & -\infty & 0 \\  -\infty  &b_2  & \cdots & -\infty & 0\\ \vdots & \vdots &\ddots &\vdots  \\ -\infty & \cdots  & -\infty & b_m & 0 \end{array}\right)\in \mbbR_{\max}^{m\times(m+1)}$. Then $\vy = B\overline{\otimes} (L\cdot \vx)$. 

\item \textbf{Leaky/Parametric ReLU activation.} For an input $\vx\in \mbbR^d$, consider an affine transformation followed by leaky ReLU units. Given $W = (W_i)\in \mbbR^{m\times d}$ where $W_i$ is the $i$-th row of $W$ and $\vb=(b_i)\in\mbbR^m$, let $\vv = (v_i)\in \mbbR^m$  with $v_i=W_i\cdot \vx+b_i$. Then the output is $\vy = (y_i)\in \mbbR^m$ computed by $y_i = \max(v_i,\lambda v_i)$.  

The above computation can also be translated to a Type I network computation. Let $L= \left(\begin{array}{c}W \\ \lambda W \end{array}\right)\in \mbbR^{(2m)\times d}$ and $$B=\left(\begin{array}{cccccccccc}b_1 & -\infty & \cdots & -\infty & \lambda b_1 & -\infty & \cdots & -\infty  \\  -\infty  &b_2  & \cdots & -\infty & -\infty  & \lambda b_2  & \cdots & -\infty\\ \vdots & \vdots &\ddots &\vdots & \vdots & \vdots &\ddots &\vdots  \\ -\infty & \cdots  & -\infty & b_m &  -\infty & \cdots  & -\infty & \lambda b_m \end{array}\right)\in \mbbR_{\max}^{m\times 2m}.$$ Then $\vy = B\overline{\otimes} (L\cdot \vx)$.

\item \textbf{Log-Sum-Exp (LSE) network.} The concept of LSE neural networks was introduced by Calafiore et al. \cite{CGP2019}  recently as a smooth universal approximator of convex functions. In addition, they show that the difference-LSE (the difference of the outputs of two LSE networks) is a universal approximator of continuous functions \cite{CGP2020}. 

By definition, an $LSE$ is a function $f:\mbbR^d\to \mbbR$ that can be written as $f(\vx) = \log\left(\sum_{i=1}^n \beta_i \exp(\va^{(i)}\cdot \vx)\right)$ for some $n\in \mbbN$. Here $\va,\vx \in\mbbR^d$. Given $T \in \mbbR_{>0}$,  by changing the base of the log function and the exp function concordantly,  an $LSE_T$ is a function that can be written as $$f_T(\vx) = T \log\left(\sum_{i=1}^n \beta_i^{1/T} \exp(\va^{(i)}\cdot \vx/T)\right)=T \log\left(\sum_{i=1}^n \exp(\va^{(i)}\cdot (\vx/T)+b_i/T)\right)$$ where $b_i = \log \beta_i$. As $T\to 0^+$, the family of functions $(f_T)_{T>0}$ converges uniformly to $f_0=\max_{i=1,\cdots,n}\left(\va^{(i)}\cdot \vx+b_i\right)$. The above procedure of deriving $f_0$ from  $(f_T)_{T>0}$ is called \emph{Maslov dequantization}, 
which is among the original motivations in the development of tropical mathematics.

Let $L = \left(\begin{array}{c}\va^{(1)} \\ \vdots \\ \va^{(n)} \end{array}\right)\in \mbbR^{n\times d}$ with $\va^{(i)}$ being row vectors. Let $B=\left(\begin{array}{ccc} b_1 & \cdots & b_n \end{array}\right)\in \mbbR^{1\times n}$. Then $f_0(\vx) = B\overline{\otimes} (L\cdot \vx)$. This means that Type I networks here are actually dequantization of $LSE$ networks. 

\end{enumerate}

\section{Universal Approximation} \label{S:UniApprox}
In this section, we show the property of universal approximation of Type I, II and III networks. In particular, we will focus our discussion on Type II networks which is the most non-conventional. Note that a Type II network contains only one linear layer and all multiplication operations of the whole network are performed in this layer. The other layers are min-plus layers and max-plus layers where only additions, min operations and max operations are performed. Actually the real power of Type II networks is that by specific configurations, one can tremendously reduce the the  computation of multiplications which are more resource-intensive than additions and min/max operations. For example, to approximate a continuous function $f:\mbbR\to \mbbR$ using a Type I network, we need to generate enough lines with distinct slopes from the linear layer to make a refined approximation of $f$. However, using a Type II network, it suffices to use much less lines with only a few slopes (for example, two fixed slopes $\pm 1$ or two  slopes derived from training) from the linear layer to  approximate $f$. 

\subsection{Type I Networks}
As shown in the previous section, Type I networks (Figure~\ref{F:MMP-NN}(b))  can be used to express several conventional networks that have already been proven to be universal approximators:

\begin{enumerate}[(i)]
\item \textbf{Maxout networks.} As shown in \cite{GWMCB2013}, maxout networks are universal approximators of convex functions and by using two maxout networks, the difference of their outputs can be used to approximate any continuous functions. 
\item \textbf{ReLU networks.} From a tropical-geometric point of view, \cite{ZNL2018} shows that the output of a ReLU network is actually a tropical rational function (the difference of two tropical posynomials). As for maxout networks, tropical posynomials are universal approximators of convex functions and  tropical rational functions are universal approximators of continuous functions. 
\item \textbf{LSE networks.} As proven in \cite{CGP2019} and \cite{CGP2020}, LSE networks are universal approximators of convex functions and difference-LSE networks are universal approximators of continuous functions. 
\end{enumerate} 

It follows that  Type I networks are universal approximators. This also means that general MMP-NNs (Figure~\ref{F:MMP-NN}(a)) are universal approximators.

\subsection{Type II networks and Linear-Min-Max networks}
As for the related networks of Type I networks discussed above, a general approach of showing the property of universal approximation of continuous functions is to show the property of universal approximation of convex functions first and then use the difference of two approximators of convex functions for the universal approximation of general continuous functions.  However, this approach does not apply to Type II networks  (Figure~\ref{F:MMP-NN}(c)),  since there is no linear layer at the output end which means that we can not derive the difference of two approximators of convex functions. 

If a Type II network only contains one single composite layer (made of a linear layer, a min-plus layer and a max-plus layer) , then we also call it a \emph{Linear-Min-Max (LMM) network}. In the following, we will analyze Type II networks in detail and show that: 
\begin{enumerate}[(1)]
\item Type II networks are equivalent to Linear-Min-Max (LMM) networks, 
\item LMM networks are universal approximators and thus Type II networks are universal approximators, and 
\item by specific configurations, the number of multiplications can be tremendously reduced in a Type II network. 
\end{enumerate}

\subsubsection{Type II networks are equivalent to Linear-Min-Max (LMM) networks}

Using the theory of tropical convexity introduced in Section~\ref{SS:tconv}, there is an elegant characterization of the space of all possible functions that can be expressed by Type II networks with the linear layer fixed. 

Consider a Typer II network which affords the function $\Phi: \mbbR^d\to \mbbR$. Then $\Phi$ can be written as a composition of transformations  $$\Phi=\beta_K\circ \alpha_K \circ  \cdots \circ \beta_2\circ \alpha_2 \circ \beta_1\circ\alpha_1\circ\lambda$$
where $\lambda$ is a linear transformation, $\alpha_i$'s are min-plus transformations and $\beta_i$'s are max-plus transformations.  Here $\beta_K$ is a max-plus linear transformation to $\mbbR$ (the output node is fixed), while the number $K$ of layers and the number of intermediate nodes are arbitrary. Denote by $S_\lambda$ the space of all functions afforded by Type II networks with the linear layer fixed to $\lambda:\mbbR^d\to\mbbR^n$.  Denote by $S'_\lambda$ the space of all functions afforded by LMM networks (meaning that $K=1$) with the linear layer fixed to $\lambda$.  

Now let the coordinates of $\mbbR^d$ be $x_1,\cdots,x_d$. Here we denote by $\tilde{x}_i$  the coordinate function on $\mbbR^d$ sending vectors $\vx = (x_j)_j \in \mbbR^d$ to the coordinate value of $x_i$ for $i=1,\cdots,d$. Then $\lambda(\vx)$ is a vector $\vy=(f_i(\vx))_i\in\mbbR^{n}$ where each $f_i:\mbbR^d\to \mbbR$ is a linear function of $\vx$. In particular, the function $f_i$ is a linear combination of the coordinate functions $\tilde{x}_1,\cdots,\tilde{x}_d$, i.e., $f_i(\vx) = a_{i1}\tilde{x}_1+\cdots+a_{id}\tilde{x}_d$. 

Let $V_\lambda:=\{f_1, \cdots, f_n\}\subseteq C(\mbbR^d)$. By Proposition~\ref{P:TropConv}, we see that

\begin{enumerate}[(i)]
\item $\lowertconv(V_\lambda)  = \{(a_1\odot g_1)\minplus\cdots\minplus (a_m\odot g_m)\mid m\in\mbbN, a_i\in \mbbR,g_i\in V_\lambda\}$
\item $\uppertconv(V_\lambda)  = \{(b_1\odot g_1)\maxplus\cdots\maxplus (b_m\odot g_m)\mid m\in\mbbN, b_i\in \mbbR,g_i\in V_\lambda\}$
\item $\lowertconv(\uppertconv(V_\lambda))=\{(a_1\odot g_1)\minplus\cdots\minplus (a_m\odot g_m)\mid m\in\mbbN, a_i\in \mbbR,g_i\in \uppertconv(V_\lambda)\}$
\item $\uppertconv(\lowertconv(V_\lambda))=\{(b_1\odot g_1)\maxplus\cdots\maxplus (b_m\odot g_m)\mid m\in\mbbN, b_i\in \mbbR,g_i\in \lowertconv(V_\lambda)\}$.
 \end{enumerate}
 
 By substitution, elements of $\lowertconv(\uppertconv(V_\lambda))$ can always be written as $(c_{11}\odot f_1\maxplus \cdots \maxplus c_{1n}\odot f_n)\minplus\cdots \minplus (c_{m1}\odot f_1\maxplus \cdots \maxplus c_{mn}\odot f_n)$ with $c_{ij}\in\mbbR_{\max}$, and elements of $\uppertconv(\lowertconv(V_\lambda))$ can always be written as $(d_{11}\odot f_1\minplus \cdots \minplus d_{1n}\odot f_n)\maxplus\cdots \maxplus (d_{m1}\odot f_1\minplus \cdots \minplus d_{mn}\odot f_n)$ with $d_{ij}\in \mbbR_{\min}$. Therefore, we may conclude that $S'_\lambda=\uppertconv(\lowertconv(V_\lambda))$.

Actually we have $\lowertconv(\uppertconv(V_\lambda)) =  \uppertconv(\lowertconv(V_\lambda))$  (a detailed proof is provided in Section~3 of our previous work \cite{Luo2018} which is  essentially based on the mutual distributive law of $\minplus$ and $\maxplus$ as stated in Remark~\ref{R:minmaxplus})), and consequently we can write $\loweruppertconv(V_\lambda):=\lowertconv(\uppertconv(V_\lambda)) =  \uppertconv(\lowertconv(V_\lambda))$. Therefore, the elements $f$ of $\loweruppertconv(V_\lambda)$ can be always be written as $(c_{11}\odot \tilde{f}_1\minplus \cdots \minplus c_{1n}\odot\tilde{f}_n)\maxplus\cdots \maxplus (c_{m1}\odot f_1\minplus \cdots \minplus c_{mn}\odot f_n)$ for some $m\in\mbbN$. This means $f$ can be realized by an LMM network $\Phi=\beta \circ\alpha \circ\lambda$ where $\alpha \in \mbbR_{\min}^{n\times m}$ is a min-plus layer and $\beta\in \mbbR_{\max}^{m\times 1}$ is a max-plus layer. Consequently, we conclude that $S'_\lambda=\loweruppertconv(V_\lambda)$.

Furthermore, for a deep Type II network $$\Phi=\beta_K\circ \alpha_K \circ  \cdots \circ \beta_2\circ \alpha_2 \circ \beta_1\circ\alpha_1\circ\lambda$$ with $K\geq 2$, we see that $$\Phi\in \uppertconv(\lowertconv(\cdots (\uppertconv(\lowertconv(V_\lambda)\cdots )= \loweruppertconv(\cdots(\loweruppertconv(V_\lambda)\cdots)$$ where $\loweruppertconv$ is applied $K$ times. But actually by applying the formula $\lowertconv(\uppertconv(V_\lambda)) =  \uppertconv(\lowertconv(V_\lambda))$ recurrently, it follows that $\loweruppertconv(\cdots(\loweruppertconv(V_\lambda)\cdots)= \loweruppertconv(V_\lambda)$. Therefore, $\Phi$ can always be equally expressed by an LMM network, i.e., $\Phi = \beta \circ\alpha \circ\lambda$ for some $m$, a min-plus transformation $\alpha \in \mbbR_{\min}^{n\times m}$ and a max-plus transformation $\beta\in \mbbR_{\max}^{m\times 1}$. 

In sum, we have the following theorem:
\begin{theorem}\label{T:TypeII-LMM}
Type II networks are equivalent to LMM networks. More specifically, for any fixed linear transformation $\lambda$, we have $S_\lambda = S'_\lambda  = \loweruppertconv(V_\lambda)$. 
\end{theorem}

\subsubsection{Type II networks are universal approximators of continuous functions}
 By Theorem~\ref{T:TypeII-LMM}, we see that to show that Type II networks are universal approximators  is equivalent to show that LMM networks are universal approximators.  
 
Recall that for a normed space $(X, \Vert \cdot \Vert )$,  a function $f:X \to \mbbR$ is called \emph{Lipschitz continuous} if there exists a constant $K>0$ such that for all $x$ and $x'$ in X, $|f(x)-f(x')| \leq K\Vert x-x'\Vert$. Any such $K$ is referred to as a \emph{Lipschitz constant} of $f$. We let $X\subseteq \mbbR^d$ and use the maximum norm  $\Vert\cdot \Vert_\infty$  on $X$ for simplicity of discussion. The following theorem actually says any Lipschitz function can be approximated by a sequence Type II networks with a common fixed linear layer. 

\begin{theorem} \label{T:UniApprox}
For any Lipschitz function $f$ on a finite region $X\subseteq \mbbR^d$, there exists some linear transformation $\lambda$  such that  a sequence of functions in $S_\lambda$ converging to $f$ uniformly.  
\end{theorem}

\begin{proof}
Let $f:X\to \mbbR$  be a Lipschitz continuous function with a Lipschiz constant $K$ for the maximum norm  $\Vert\cdot  \Vert_\infty$  on $\mbbR^d$. This means that $|f(\vx)-f(\vx')| \leq K\Vert \vx-\vx'\Vert_\infty$ for all $\vx,\vx'\in X$. We will use LMM networks to make the approximation in the following. 

For the linear layer, we use a fixed linear transformation $\lambda: \mbbR^d \to \mbbR^{2d}$ where $\lambda(\vx) =  (f_i(\vx))_i\in\mbbR^{2d}$ with $f_{2i-1}(\vx)=Kx_i$ and $f_{2i}(\vx)=-Kx_i$ for $\vx =(x_i)_i\in\mbbR^d$. Then after a min-plus linear transformation $\alpha:\mbbR^{2d}\to \mbbR^m$, we get $\alpha(\lambda(\vx)) = (g_i(\vx))_i\in \mbbR^m$ where $g_i(\vx) = a_{i1}\odot f_1(\vx)\minplus\cdots \minplus a_{i,2d}\odot f_{2d}(\vx) = \min(f_1(\vx)+a_{i1},\cdots, f_{2d}(\vx)+a_{i,2d})$ for $i=1,\cdots,m$. We note that the graph of $g_i$ as a function on $\mbbR^d$ has the shape of a pyramid. Let $P_i$ be the tip of the pyramid of $g_i$. By varying value of $a_{ij}$, the tip $P_i$ can have an arbitrary height and the projection of $P_i$ to $\mbbR^d$ can be located anywhere in $\mbbR^d$. In particular, we call the projection of  $P_i$ the center of $g_i$. 

Consider two points $\vx$ and $\vx'$ in $\mbbR^d$. Let $g$ and $g'$ be two min-plus combinations of $f_1,\cdots, f_{2d}$ with centers at $\vx$ and $\vx'$ respectively. 

Let $h=\max(g,g')$.  We observe that as long as $|g(\vx)-g'(\vx')|\leq K \Vert \vx - \vx' \Vert_\infty$, we have $h(\vx)=g(\vx)$ and $h(\vx')=g'(\vx')$. It is easy to verify that  $g$, $g'$ and $h$ are Lipschitz continuous also with a Lipschiz constant $K$. 

Assigning a grid $\mcalG$ on $\mbbR^d$ with $\Delta x_1=\cdots \Delta x_n = \delta$.  Since $X$ is a finite region, the intersection of $\mcalG$ and $X$ must be a finite set $M$. Suppose the cardinality of $M$ is $m$. By adjusting the entries for the matrix $(a_{ij})\in\mbbR^{m\times(2d)}$, we can arrange the $g_i$'s such that  the following properties are satisfied: 
\begin{enumerate}[(i)]
\item The set of centers of $g_i$'s is exactly $M$; and
\item for each $g_i$, we have $g_i(\vx_i) = f(\vx_i)$ where $\vx_i$ is the center of $g_i$. 
\end{enumerate}

Let $h=\max_{i=1,\cdots m}(g_i)$. We claim that for each point $\vx\in X$, $|h(\vx) - f(\vx)|$ is bounded by $2K\delta$. Actually let $\vx'$ be a grid point nearest to $\vx$. This means that $\Vert\vx-\vx'\Vert_\infty\leq \delta$. Therefore, $|h(\vx) - h(\vx')|\leq K\delta$ and $|f(\vx) - f(\vx')|\leq K\delta$ by the Lipschitz continuity of $h$ and $f$. In addition, since $g_i(\vx_i) = f(\vx_i)$ for all the grid points $\vx_i$ and $f$ is Lipschitz continuous, we conclude that $h(\vx_i) = g_i(\vx_i)=f(\vx_i)$ for all grid points $\vx_i$. Then $$|h(\vx) - f(\vx)|\leq |h(\vx) - h(\vx')|+|h(\vx') - f(\vx')|+|f(\vx') - f(\vx)| \leq K\delta + 0+ K\delta = 2K\delta.$$

The above argument actually shows that $\Vert h-g\Vert_\infty\leq 2K\delta$. Therefore, we see that $h$ approaches $f$ uniformly as $\delta$ approaches $0$ which can be achieved by increasing the number of intermediate nodes of the LMM network. Thus we have proved that LMM networks (equivalently Type II networks) with fixed linear layer are universal approximators of Lipschitz  functions. 
\end{proof}

Note that on a compact metric space $X$, any continuous function is the uniform limit of a sequence of Lipschitz functions \cite{Georgano1967}. Hence we have the following corollary of the above theorem. 

\begin{corollary}\label{C:UniApprox}
Type II networks (equivalently LMM networks) are universal approximators of continuous functions. 
\end{corollary}

\subsubsection{Cutting down the number  of multiplications in Type II networks}

In the proof of Theorem~\ref{T:UniApprox}, if the input of the Type II network is $d$-dimensional, there are only $d$ multiplications in total that need to be computed ($f_{2i}(\vx)=-Kx_i$ can be computed from $f_{2i-1}(\vx)=Kx_i$ by negation), all in the linear transformation $\lambda: \mbbR^d \to \mbbR^{2d}$. If we already know that the function to be approximated has relatively small variations (in this case we may assume $K=1$ is a Lipschitz constant of $f$), then we may even simply fix the linear layer with $x_i$ affording $f_{2i-1}(\vx) = x_i$ and $f_{2i}(\vx) = -x_i$. In this case, no multiplication is necessary. 

Moreover, to form pyramid shaped functions after the min-plus activation, we just need a linear transformation $\mbbR^d \to \mbbR^{d+1}$  instead of $\mbbR^d \to \mbbR^{2d}$ as in in the proof of Theorem~\ref{T:UniApprox}, i.e., only $d+1$ hyperplanes need to be generated from the linear layer instead of $2d$ hyperplanes as shown previously  (for example, we may let $f_i(\vx)=Kx_i$ for $i=1,\cdots,d$ and $f_{d+1}=-K(x_1+\cdots+x_d)$). This number can be further reduced in real applications.

On the other hand, as a tradeoff,  if more multiplications are introduced in the, then the approximating function generated by  the network can be ``smoother''.

The above discussion of Type II networks as function approximators also implies that to train a Type II network, one may either fix the linear layer with preassigned linear transformations and only train the min-plus and max-plus layers (in this case, each component of the network output is a Lipschitz  function whose Lipschitz constants have an upper bound which is determined by the preassigned linear transformation), or train the linear layer together with the min-plus and max-plus layers (in this case, each component of the network output can approximate any continuous function by Corollary~\ref{C:UniApprox}).

\subsection{Type III networks and Linear-Min-Max-Linear networks}
Note that a Type III network is a Type II network attached with an additional linear layer to the output end. Then the fact that Type II networks are equivalent to LMM networks (Theorem~\ref{T:TypeII-LMM}) implies that Type III networks are equivalent to\emph{ Linear-Min-Max-Linear (LMML) networks} which are LMM networks attached with an additional linear layer to the output end. 

For a set $V$ of vectors, denote the linear span of $V$ by $\Span (V):=\lbrace\sum_{i=1}^kc_i\vv_i\mid k\in\mbbN, c_i\in\mbbR, \vv_i\in V\rbrace$. Let $\Lambda_d$ be the space of all linear transformations on $\mbbR^d$. Then the space of all functions expressed by Type II networks with $d$ input nodes is $\Expr_d^{\text{\Romannum{2}}}:=\bigcup_{\lambda\in \Lambda_d}S_\lambda$, and the space of all functions expressed by Type III networks with $d$ input nodes is $\Expr_d^{\text{\Romannum{3}}}:=\Span\left(\Expr_d^{\text{\Romannum{2}}}\right)$.

It is clear that Type III networks are universal approximators of continuous functions, since Type II networks are universal approximators of continuous functions. On the other hand, compared to most conventional neural networks with simple fixed nonlinear activation functions, one may consider Type III networks as conventional networks with the nonlinear activation part being enlarged and trainable. Note that the composition of min-plus and max-plus layers afford a powerful nonlinear expression capability, which by itself afford Type II networks (not containing the output linear layer as Type III networks) capability of universal approximation of continuous functions. One may expect that the more sophisticated nonlinear expressor in a Type III network enhances its overall fitting capability as compared to conventional networks of similar scale.

\section{Training MMP-NNs} \label{S:train}
In this section, we first formulate the backpropagation algorithm for training the three types of layers of MMP-NNs, and then introduce a special technique called ``normalization''/``restricted normalization'' by which the parameters in the nonlinear part of MMP-NNs (entries of the min-plus and max-plus matrices) are properly adjusted which counters the effect of parameter deviation which can commonly happen to the nonlinear part of  MMP-NNs and might seriously affect the convergence of training. In particular, the normalization technique can be considered as a generalization of  is inspired by the  the Legendre-Fenchel transformation or convex conjugate which is widely used in physics and convex optimization. 

\subsection{Backpropagation} \label{SS:Backprop}
As for the conventional feedforward neural networks, we can use backpropagation to train MMP-NNs. 

In the process of backpropagation, the calculation methods of parameter gradients of min-plus layers and max-plus layers is different from that of linear layers. We summarize the gradient calculations as follows. 

\begin{enumerate}[(i)]
\item For the linear layers, consider linear functions
$$y_i = \sum_{j=1}^{n} w_{i j} \cdot x_j$$
where $n$ is the dimension of input.

From the forward calculation formula of linear layer, the gradient calculation formula of $w$ can be directly obtained as follows:
$$
\Delta w_{i j}=\frac{\partial y_{i}}{\partial w_{i j}} \cdot \Delta y_{i}= x_{j } \cdot \Delta y_{i}
$$
where $\Delta y_{i}$ is the loss of the output $y_i$.

\item As for the min-plus layers, the forward calculation formula can be described as:

$$
y'_i = \min\left(w'_{i 1} + x'_1,...,w'_{i j} + x'_j,...,w'_{i n} + x'_{n}\right).
$$

And we get the gradient calculation formula of $w$ as the following:

$$
\Delta w'_{i j}=\frac{\partial y'_{i}}{\partial w'_{i j}} \cdot \Delta y'_{i}=\left\{\begin{array}{cc}
\Delta y'_{i} & , (1) \\
0 & , (2)
\end{array}\right.
$$

(1) when the term of $w'_{i j}$ is the smallest term in the forward propagation used to calculate $y'_i$.

(2) otherwise.

\item Analogously, the forward calculation formula of the max-plus layers can be described as:

$$
y''_i = \max\left(w''_{i 1} + x'_1,...,w''_{i j} + x''_j,...,w''_{i n} + x''_{n}\right).
$$

And the gradient calculation formula:

$$
\Delta w''_{i j}=\frac{\partial y''_{i}}{\partial w''_{i j}} \cdot \Delta y''_{i}=\left\{\begin{array}{cc}
\Delta y''_{i} & , (3) \\
0 & , (4)
\end{array}\right.
$$

(3) when the term of $w''_{i j}$ is the biggest term in the forward propagation used to calculate $y''_i$.

(4) otherwise.

\end{enumerate}

\subsection{Normalization and restricted normalization} \label{SS:normalization}

Consider a min-plus transformation $\alpha: \mbbR^n\to \mbbR^m$ and a max-plus transformation $\beta: \mbbR^n\to \mbbR^m$ represented by $\alpha(\vx) = A\underline{\otimes} \vx$ and $\beta(\vx) = B\overline{\otimes} \vx$ respectively where $\vx\in \mbbR^n$, $A=(a_{ij})_{ij}\in\mbbR_{\min}^{m\times n}$ and $B=(b_{ij})_{ij}\in \mbbR_{\max}^{m\times n}$. Let $f_1,\cdots,f_n$ be real-valued functions on the same domain $X$. Then for $i=1,\cdots,m$, the min-plus combinations $g_i=a_{i1}\odot f_1\minplus\cdots\minplus a_{in}\odot f_n=\min\left(a_{i1}+f_1,\cdots,a_{in}+f_n\right)$ and the max-plus combinations $h_i=b_{i1}\odot f_1\maxplus\cdots\maxplus b_{in}\odot f_n=\max\left(b_{i1}+f_1,\cdots,b_{in}+f_n\right)$ are also real-valued functions on $X$ (Figure~\ref{F:min-max-transform}). (In this sense, $\alpha$ and $\beta$ can also be considered respectively as the min-plus and max-plus transformations of the vector $(f_1,\cdots,f_n)$ of functions on $X$.)

Consider, for example, $g_i=\min\left(a_{i1}+f_1,\cdots,a_{in}+f_n\right)$. Note that for a specific $x\in X$, $g_i(x)$ must take the same value as $a_{ij}+f_j(x)$ for at least one of the $j$'s in $\{1,\cdots,n\}$. On the other hand, however, it is possible that for some $f_j$, $ \inf_{x\in X}\left((a_{ij}+f_j(x))-g_i(x)\right)>0$ which means that the graph of the function $a_{ij}+f_j$ is strictly above the graph of the function $g_i$ by a margin. If such a ``detachment'' of the functions $a_{ij}+f_j$ and $g_i$ happens, a small adjustment of the coefficient $a_{ij}$ in the training process describe in Subsection~\ref{SS:Backprop} will be noneffective which can seriously affects the convergence of training. 

To avoid such a  phenomenon of ``parameter deviation'' and expedite the convergence of training,  we introduce a normalization process of the parameters as described in below.

\begin{figure}
\centering
\begin{tikzpicture}[>=to,x=1cm,y=1cm, scale=1]

\begin{scope}[shift={(0cm, 0cm)}]
  \node[scale=1.2] at (-2.2,0.7) {$f_1$};
  \node[scale=1.2] at (-1.8,0.1) {$\vdots$};
  \node[scale=1.2] at (-2.2,-0.7) {$f_n$};
  \node[scale=1.2] at (2.2,0.7) {$g_1$};
  \node[scale=1.2] at (1.8,0.1) {$\vdots$};
  \node[scale=1.2] at (2.2,-0.7) {$g_m$};
  \node[scale=1.2] at (0,0) {$A\in\mbbR^{m\times n}_{\min}$};
  \draw[-,line width=.8pt] (-1.8,0.7) edge (-1.3,0.7);
  \draw[-,line width=.8pt] (-1.8,-0.7) edge (-1.3,-0.7);
  \draw[-,line width=.8pt] (1.3,0.7) edge (1.8,0.7);
  \draw[-,line width=.8pt] (1.3,-0.7) edge (1.8,-0.7);
  \draw[color=black] (-1.3,-1.2) rectangle (1.3,1.2) ;
\end{scope}

\begin{scope}[shift={(6cm, 0cm)}]
  \node[scale=1.2] at (-2.2,0.7) {$f_1$};
  \node[scale=1.2] at (-1.8,0.1) {$\vdots$};
  \node[scale=1.2] at (-2.2,-0.7) {$f_n$};
  \node[scale=1.2] at (2.2,0.7) {$h_1$};
  \node[scale=1.2] at (1.8,0.1) {$\vdots$};
  \node[scale=1.2] at (2.2,-0.7) {$h_m$};
  \node[scale=1.2] at (0,0) {$B\in\mbbR^{m\times n}_{\max}$};
  \draw[-,line width=.8pt] (-1.8,0.7) edge (-1.3,0.7);
  \draw[-,line width=.8pt] (-1.8,-0.7) edge (-1.3,-0.7);
  \draw[-,line width=.8pt] (1.3,0.7) edge (1.8,0.7);
  \draw[-,line width=.8pt] (1.3,-0.7) edge (1.8,-0.7);
  \draw[color=black] (-1.3,-1.2) rectangle (1.3,1.2) ;
\end{scope}

\end{tikzpicture}
\caption{Functions $f_1,\cdots,f_n$ are transformed to $g_1,\cdots,g_m$ by the min-plus matrix $A\in\mbbR^{m\times n}_{\min}$ and to $h_1,\cdots,h_m$ by the max-plus matrix $B\in\mbbR^{m\times n}_{\max}$.} \label{F:min-max-transform}
\end{figure}
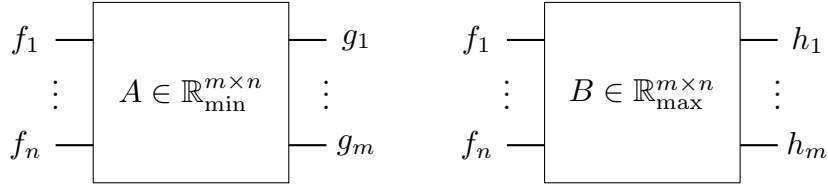

\begin{algorithm} \label{A:normalization}
\textbf{(Normalization)}
\begin{enumerate}[(1)]
\item \textbf{Min-plus normalization} 
\begin{enumerate}
\item[\textbf{Input:}] A min-plus matrix $A=(a_{ij})_{ij}\in\mbbR_{\min}^{m\times n}$ and functions $f_1,\cdots,f_n\in C(X)$.
\item[\textbf{Computation:}] 
For $i=1,\cdots,m$, let $g_i=a_{i1}\odot f_1\minplus\cdots\minplus a_{in}\odot f_n=\min\left(a_{i1}+f_1,\cdots,a_{in}+f_n\right)\in C(X)$.
For each $i=1,\cdots,m$ and $j=1,\cdots,n$, compute
$$\underline{\nu}(a_{ij})=-\inf_{x\in X}\left(f_j(x)-g_i(x)\right)=\sup_{x\in X}\left(g_i(x)-f_j(x)\right).$$
(In most of the cases we are interested, infimum and supremum can be replaced by minimum and maximum respectively.)

\item[\textbf{Output:}]  A min-plus matrix $\underline{\nu}(A):=(\underline{\nu}(a_{ij}))_{ij}\in \mbbR_{\min}^{m\times n}$.

\end{enumerate}
\item \textbf{Max-plus normalization}
\begin{enumerate}
\item[\textbf{Input:}] A max-plus matrix $B=(b_{ij})_{ij}\in\mbbR_{\max}^{m\times n}$ and functions $f_1,\cdots,f_n\in C(X)$
\item[\textbf{Computation:}] 
For $i=1,\cdots,m$, let $h_i=b_{i1}\odot f_1\maxplus\cdots\maxplus b_{in}\odot f_n=\max\left(b_{i1}+f_1,\cdots,b_{in}+f_n\right)\in C(X)$.
For each $i=1,\cdots,m$ and $j=1,\cdots,n$, compute
$$\overline{\nu}(b_{ij})=-\sup_{x\in X}\left(f_j(x)-h_i(x)\right)=\inf_{x\in X}\left(h_i(x)-f_j(x)\right).$$
(In most of the cases we are interested, infimum and supremum can be replaced by minimum and maximum respectively.)
\item[\textbf{Output:}] A max-plus matrix $\overline{\nu}(A):=(\overline{\nu}(b_{ij}))_{ij}\in \mbbR_{\max}^{m\times n}$.
\end{enumerate}
\end{enumerate}

\end{algorithm}

\begin{remark} \label{R:normalization}
We call $\underline{\nu}(a_{ij})$ the \emph{(min-plus) normalization of $a_{ij}$}, $\underline{\nu}(A)$ the \emph{(min-plus) normalization of $A$}, $\overline{\nu}(b_{ij})$ the \emph{(max-plus) normalization of $b_{ij}$}, and $\overline{\nu}(B)$ the \emph{(max-plus) normalization of $B$}. 
\end{remark}

\begin{proposition} \label{P:normalization}
Using the notations in Algorithm~\ref{A:normalization}, the normalization has the following properties:
\begin{enumerate}[(a)]
\item For each $i=1,\cdots,m$ and $j=1,\cdots,n$, $\inf_{x\in X}((\underline{\nu}(a_{ij})+f_j(x))-g_i(x))=0$ and $\sup_{x\in X}((\overline{\nu}(b_{ij})+f_j(x))-h_i(x))=0$.
\item $\underline{\nu}(a_{ij})\leq a_{ij}$ and $\overline{\nu}(b_{ij})\geq b_{ij}$.
\item  For each $i=1,\cdots,m$, $$g_i=\underline{\nu}(a_{i1})\odot f_1\minplus\cdots\minplus \underline{\nu}(a_{in})\odot f_n=\min\left(\underline{\nu}(a_{i1})+f_1,\cdots,\underline{\nu}(a_{in})+f_n\right)$$ and 
$$h_i=\overline{\nu}(b_{i1})\odot f_1\maxplus\cdots\maxplus \overline{\nu}(b_{in})\odot f_n=\max\left(\overline{\nu}(b_{i1})+f_1,\cdots,\overline{\nu}(b_{in})+f_n\right).$$
\end{enumerate}
\end{proposition}
\begin{proof}
Here we will only prove the case of min-plus normalization. The case of max-plus normalization can be proved analogously. 

For (a), since $\underline{\nu}(a_{ij})=\sup_{x\in X}\left(g_i(x)-f_j(x)\right)$, we must have that (1) $g_i(x)\leq \underline{\nu}(a_{ij})+f_j(x)$ for all $x\in X$ and (2) for each $\delta>0$, there exists $x\in X$ such that $g_i(x)+\delta \geq \underline{\nu}(a_{ij})+f_j(x)$. This means that $\inf_{x\in X}((\underline{\nu}(a_{ij})+f_j(x))-g_i(x))=0$.

For (b), note that $g_i=\min\left(a_{i1}+f_1,\cdots,a_{in}+f_n\right)$. Hence $g_i\leq a_{ij}+f_j$ which implies $\inf_{x\in X}((a_{ij}+f_j(x))-g_i(x))\geq 0$. Therefore,  we must have $\underline{\nu}(a_{ij})\leq a_{ij}$  by (a). 

For (c), we observe that $\min\left(\underline{\nu}(a_{i1})+f_1,\cdots,\underline{\nu}(a_{in})+f_n\right)\leq g_i=\min\left(a_{i1}+f_1,\cdots,a_{in}+f_n\right)$ by (b). On the other hand, for each $j$, we must have $g_i(x)\leq \underline{\nu}(a_{ij})+f_j(x)$ for all $x\in X$ since  $\underline{\nu}(a_{ij})=\sup_{x\in X}\left(g_i(x)-f_j(x)\right)$. Therefore, we have the identity $g_i=\min\left(\underline{\nu}(a_{i1})+f_1,\cdots,\underline{\nu}(a_{in})+f_n\right)$. 
\end{proof}

\begin{example} \label{E:normalization}

\begin{figure}
\centering
\begin{tikzpicture}[
  declare function={
    g(\x)= (\x<=-1) * (\x+2)   +
     and(\x>-1, \x<=1) * (\x*\x)     +
     (\x>1) * (-\x+2);
  }
]
\begin{scope} [shift={(0,0)}]

\node[scale=1.2] at (0,7) {(a)};

\begin{axis}[
	width=6.5cm,
	height=6.5cm,
    scale only axis,
    xmin=-2, xmax=2,
    xtick={-2,-1,0,1,2},ytick={0,1,2,3}, 
    xlabel = $x$,
    ymin=-0.5, ymax=3.5,
]

\addplot [
	line width=0.6pt,
    domain=-2:2, 
    samples=100, 
    color=red,
]
{x+2};
\addlegendentry{$f_1(x)+2=x+2$}

\addplot [
	line width=0.6pt,
    domain=-2:2, 
    samples=100, 
    color=green,
]
{-x+2};
\addlegendentry{$f_2(x)+2=-x+2$}

\addplot [
	line width=0.6pt,
    domain=-2:2, 
    samples=100, 
    color=yellow,
]
{x^2};
\addlegendentry{$f_3(x)=x^2 $}

\addplot [
	line width=0.6pt,
    domain=-2:2, 
    samples=100, 
    color=blue,
]
{1.5};
\addlegendentry{$f_4(x)+1.5=1.5$}

\addplot [
	line width=2.5pt,
	opacity=0.6,
    domain=-2:2, 
    samples=100, 
    color=black,
]
{g(x)};
\addlegendentry{$g(x)$}

\end{axis}
\end{scope}

\begin{scope} [shift={(8,0)}]
\node[scale=1.2] at (0,7) {(b)};

\begin{axis}[
	width=6.5cm,
	height=6.5cm,
    scale only axis,
    xmin=-2, xmax=2,
    xtick={-2,-1,0,1,2},ytick={0,1,2,3}, 
    xlabel = $x$,
    ymin=-0.5, ymax=3.5,
]

\addplot [
	line width=0.6pt,
    domain=-2:2, 
    samples=100, 
    color=red,
]
{x+2};
\addlegendentry{$f_1(x)+2=x+2$}

\addplot [
	line width=0.6pt,
    domain=-2:2, 
    samples=100, 
    color=green,
]
{-x+2};
\addlegendentry{$f_2(x)+2=-x+2$}

\addplot [
	line width=0.6pt,
    domain=-2:2, 
    samples=100, 
    color=yellow,
]
{x^2};
\addlegendentry{$f_3(x)=x^2 $}

\addplot [
	line width=0.6pt,
    domain=-2:2, 
    samples=100, 
    color=blue,
]
{1};
\addlegendentry{$f_4(x)+1=1$}

\addplot [
	line width=2.5pt,
	opacity=0.6,
    domain=-2:2, 
    samples=100, 
    color=black,
]
{g(x)};
\addlegendentry{$g(x)$}

\end{axis}
\end{scope}
\end{tikzpicture}
\caption{An example of one-dimensional min-plus normalization:  (a) A min-plus combination of functions $f_1(x)=x$, $f_2(x)=-x$, $f_3(x)=x^2$ and $f_4(x)=0$ as $g(x)=\min(f_1(x)+2,f_2(x)+2,f_3(x),f_4(x)+1.5)$; (b) After normalization, we can write $g(x)=\min(f_1(x)+2,f_2(x)+2,f_3(x),f_4(x)+1)$.}  \label{F:normalization}
\end{figure}
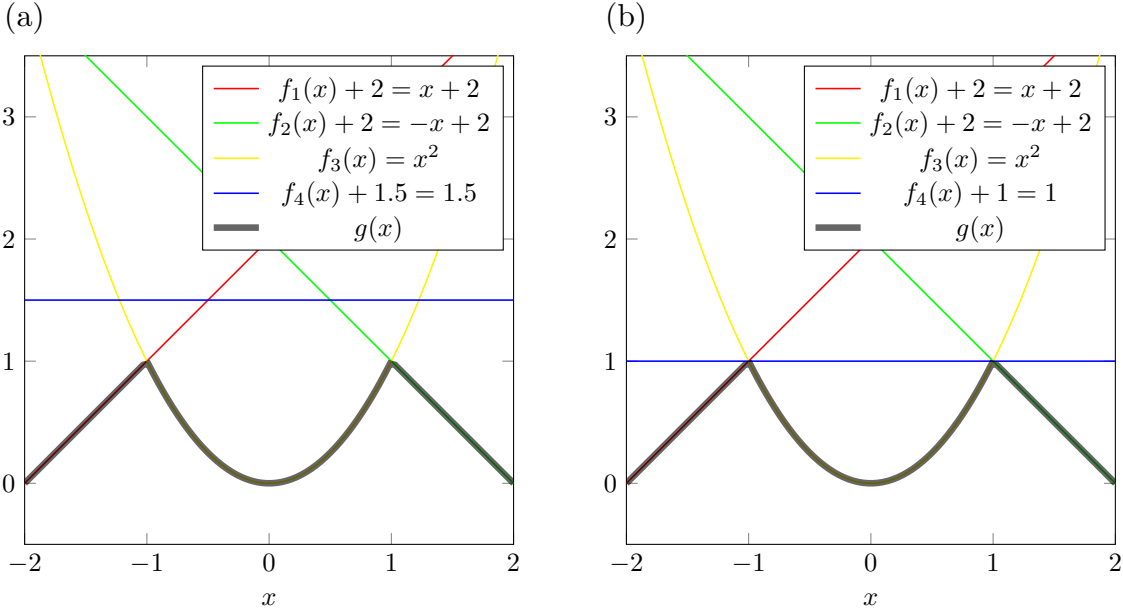
For  $X=\mbbR$, let $f_1(x)=x$, $f_2(x)=-x$, $f_3(x)=x^2$ and $f_4(x)=0$ be functions on $\mbbR$. Consider a function $g=2\odot f_1\minplus 2\odot f_2\minplus f_3\minplus 1.5\odot f_4 = \min(f_1+2,f_2+2,f_3,f_4+1.5)$ which is a min-plus combination of $f_1$, $f_2$, $f_3$ and $f_4$. As shown in Figure~\ref{F:normalization}(a), $g(x)=f_1(x)+2$ for $x\leq -1$, $g(x)=f_3(x)$ for $1\leq x\leq 1$ and $g(x)=f_2(x)+2$ for $x\geq 1$. However, $\min((f_4+1.5)-g)=0.5$, i.e., the whole graph of $f_4+1.5$ is detached from the graph of $g$. Running the min-plus normalization algorithm, we derive $\max(g-f_1)=\max(g-f_2)=2$, $\max(g-f_3)=0$, and $\max(g-f_4)=1$. Therefore, after normalization, we can write $g=\min(f_1+2,f_2+2,f_3,f_4+1)$ while the graphs of $f_1+2$, $f_2+2$, $f_3$ and $f_4+1$ are all attached to the graph of $g$ as shown in Figure~\ref{F:normalization}(b). 
\end{example}

\begin{remark}\label{R:legendre}
We emphasize here that if the functions $f_1,\cdots,f_n$ in Algorithm~\ref{A:normalization} are linear functions on a Euclidean space, then the normalization process described in Algorithm~\ref{A:normalization} is exactly the classical Legendre transformation applied to convex functions. On the other hand, there is no restriction of being linear on the functions $f_1,\cdots,f_n$ in our algorithm, while the normaliation process is generally true. In our previous work \cite{Luo2018}, we've developed a theory of tropical convexity analysis providing a more refined treatment of this general setting. 
\end{remark}

\begin{remark} \label{R:independence}
Recall that  a general MMP-NN of $K$ composite layer with  $d$  input nodes and $p$ output nodes of the MMP-NN  is essentially a function $\Phi: \mbbR^d\to \mbbR^p$ which can be written as $\Phi=\beta_K\circ \alpha_K\circ \lambda_K \circ \beta_{K-1} \circ \alpha_{K-1} \circ \lambda_{K-1}\circ \cdots \circ \beta_1\circ\alpha_1\circ\lambda_1$ where $\lambda_i$, $\alpha_i$ and $\beta_i$ are linear, min-plus and max-plus transformations respectively. In a training process, we may apply min-plus normalizations to the min-plus layers $\alpha_i$ and max-plus normalizations to the max-plus layers $\beta_i$ after steps of parameter tuning to improve and expedite convergence. In particular, for example, the input of $\alpha_i$ is a vector of functions on $\mbbR^d$ and the output of $\alpha_i$ is min-plus combinations of the input functions. If we apply min-plus normalization to $\alpha_i$, the parameters which are entries to the corresponding min-plus matrix will be adjusted. However, by Proposition~\ref{P:normalization}, this parameter adjustment won't affect the output functions. This means that the normalization process of one layer won't affect the results of other layers. The advantage of this layer-to-layer independence of normalization is that if we need to normalize one specific layer, there is no need to process the preceding layers first. 
\end{remark}

In Algorithm~\ref{A:normalization}, a practical issue of  computing $\underline{\nu}(a_{ij})$ and$\overline{\nu}(b_{ij})$ is that  we need to find the minimum/maximum of some functions on $X$ which is usually a domain in a Euclidean space. The computation of the minimum/maximum of a function can be extremely hard and time-consuming itself unless the input functions are simple functions and the domain $X$ is also simple of low dimension. In a MMP-NN, typically only the input functions of the first min-plus layer $\alpha_1$ are linear functions enabling us to find rigorous solutions of the minimum/maximum for layer $\alpha_1$ in a relatively easy manner. the input functions of the min-plus/max-plus layers following $\alpha_1$ become more untamed as the layer goes deeper. As a result, a direct min/max computation is a lot harder for these layers. 

To solve this problem, we propose a modified version of Algorithm~\ref{A:normalization} called restricted normalization (Algorithm~\ref{A:restricted-normalization}) which only deals with a finite subset $D$ of the domain $X$. In practice, $D$ can be simply chosen from the training set which tremendously simplifies the min/max computation in the normalization process.

\begin{algorithm}\label{A:restricted-normalization}
\textbf{(Restricted Normalization)}
\begin{enumerate}[(1)]
\item \textbf{Restricted min-plus normalization} 
\begin{enumerate}
\item[\textbf{Input:}] A min-plus matrix $A=(a_{ij})_{ij}\in\mbbR_{\min}^{m\times n}$, functions $f_1,\cdots,f_n\in C(X)$, and a finite subset $D\subseteq X$.
\item[\textbf{Computation:}] 
For $i=1,\cdots,m$, let $g_i=a_{i1}\odot f_1\minplus\cdots\minplus a_{in}\odot f_n=\min\left(a_{i1}+f_1,\cdots,a_{in}+f_n\right)\in C(X)$.
For each $i=1,\cdots,m$ and $j=1,\cdots,n$, compute
$$\underline{\nu}_D(a_{ij})=-\min_{x\in D}\left(f_j(x)-g_i(x)\right)=\max_{x\in D}\left(g_i(x)-f_j(x)\right).$$

\item[\textbf{Output:}]  A min-plus matrix $\underline{\nu}_D(A):=(\underline{\nu}_D(a_{ij}))_{ij}\in \mbbR_{\min}^{m\times n}$.

\end{enumerate}
\item \textbf{Restricted max-plus normalization}
\begin{enumerate}
\item[\textbf{Input:}] A max-plus matrix $B=(b_{ij})_{ij}\in\mbbR_{\max}^{m\times n}$ and functions $f_1,\cdots,f_n\in C(X)$,  and a subset $D\subseteq X$.
\item[\textbf{Computation:}] 
For $i=1,\cdots,m$, let $h_i=b_{i1}\odot f_1\maxplus\cdots\maxplus b_{in}\odot f_n=\max\left(b_{i1}+f_1,\cdots,b_{in}+f_n\right)\in C(X)$.
For each $i=1,\cdots,m$ and $j=1,\cdots,n$, compute
$$\overline{\nu}_D(a_{ij})=-\max_{x\in D}\left(f_j(x)-h_i(x)\right)=\min_{x\in D}\left(h_i(x)-f_j(x)\right).$$
\item[\textbf{Output:}] A max-plus matrix $\overline{\nu}_D(A):=(\overline{\nu}_D(b_{ij}))_{ij}\in \mbbR_{\max}^{m\times n}$.
\end{enumerate}
\end{enumerate}
\end{algorithm}

\begin{remark} \label{R:restricted-normalization}
We call $\underline{\nu}_D(a_{ij})$ the \emph{(min-plus) normalization of $a_{ij}$ restricted to $D$}, $\underline{\nu}_D(A)$ the \emph{(min-plus) normalization of $A$ restricted to $D$}, $\overline{\nu}_D(b_{ij})$ the \emph{(max-plus) normalization of $b_{ij}$ restricted to $D$}, and $\overline{\nu}_D(B)$ the \emph{(max-plus) normalization of $B$ restricted to $D$}. 
\end{remark}

\begin{proposition} \label{P:restricted-normalization}
Using the notations in Algorithm~\ref{A:normalization} and Algorithm~\ref{A:restricted-normalization},  for $i=1,\cdots,m$, let 
$$g'_i=\underline{\nu}_D(a_{i1})\odot f_1\minplus\cdots\minplus \underline{\nu}_D(a_{in})\odot f_n=\min\left(\underline{\nu}_D(a_{i1})+f_1,\cdots,\underline{\nu}_D(a_{in})+f_n\right)$$ 
and 
$$h'_i=\overline{\nu}_D(b_{i1})\odot f_1\maxplus\cdots\maxplus \overline{\nu}_D(b_{in})\odot f_n=\max\left(\overline{\nu}_D(b_{i1})+f_1,\cdots,\overline{\nu}_D(b_{in})+f_n\right).$$ The restricted normalization has the following properties:
\begin{enumerate}[(a)]
\item For each $i=1,\cdots,m$ and $j=1,\cdots,n$, $\min_{x\in X}((\underline{\nu}_D(a_{ij})+f_j(x))-g'_i(x))=\min_{x\in D}((\underline{\nu}_D(a_{ij})+f_j(x))-g'_i(x))=0$ and $\max_{x\in X}((\overline{\nu}_D(b_{ij})+f_j(x))-h'_i(x))=\max_{x\in D}((\overline{\nu}_D(b_{ij})+f_j(x))-h'_i(x))=0$.
\item $\underline{\nu}_D(a_{ij})\leq\underline{\nu}(a_{ij})\leq a_{ij}$ and $\overline{\nu}_D(b_{ij})\geq \overline{\nu}_D(b_{ij})\geq b_{ij}$.
\item  For each $i=1,\cdots,m$, we have $g'_i\leq g_i$ and $h'_i\geq h_i$ while $g'_i|_D=g_i|_D$ and  $h'_i|D=h_i|_D$.
\end{enumerate}
\end{proposition}

\begin{proof}
Here we will only prove the case of restricted min-plus normalization. The case of restricted max-plus normalization can be proved analogously. 

Consider all the functions restricted to $D$. Replacing $X$ by $D$, we can use the same argument as in the proof of Proposition~\ref{P:normalization} to show that $\min_{x\in D}((\underline{\nu}_D(a_{ij})+f_j(x))-g'_i(x))=0$, $\underline{\nu}_D(a_{ij})\leq a_{ij}$ and $g'_i|_D=g_i|_D$. 

To show that $\min_{x\in X}((\underline{\nu}_D(a_{ij})+f_j(x))-g'_i(x))=\min_{x\in D}((\underline{\nu}_D(a_{ij})+f_j(x))-g'_i(x))$, we note that by definition of $g'$, we have $g'_i(x)\leq \underline{\nu}_D(a_{ij})+f_j(x)$ for all $j$ and all $x\in X$, which means that $\min_{x\in X}((\underline{\nu}_D(a_{ij})+f_j(x))-g'_i(x))\geq 0$. On the other hand, the minimum can be achieved at some point $x$ in D. Hence the identity follows. 

It remains to show that $\underline{\nu}_D(a_{ij})\leq\underline{\nu}(a_{ij})$   and $g'_i\leq g_i$. By definition of $\underline{\nu}_D(a_{ij})$ and $\underline{\nu}(a_{ij})$, we have $\underline{\nu}_D(a_{ij})=\max_{x\in D}\left(g_i(x)-f_j(x)\right)\leq \sup_{x\in X}\left(g_i(x)-f_j(x)\right)=\underline{\nu}(a_{ij})$. Using Proposition~\ref{P:normalization}~(c), this also implies that $g'=\min\left(\underline{\nu}_D(a_{i1})+f_1,\cdots,\underline{\nu}_D(a_{in})+f_n\right)\leq \min\left(\underline{\nu}(a_{i1})+f_1,\cdots,\underline{\nu}(a_{in})+f_n\right)=g$. 

\end{proof}

\begin{example} \label{E:restricted-normalization}
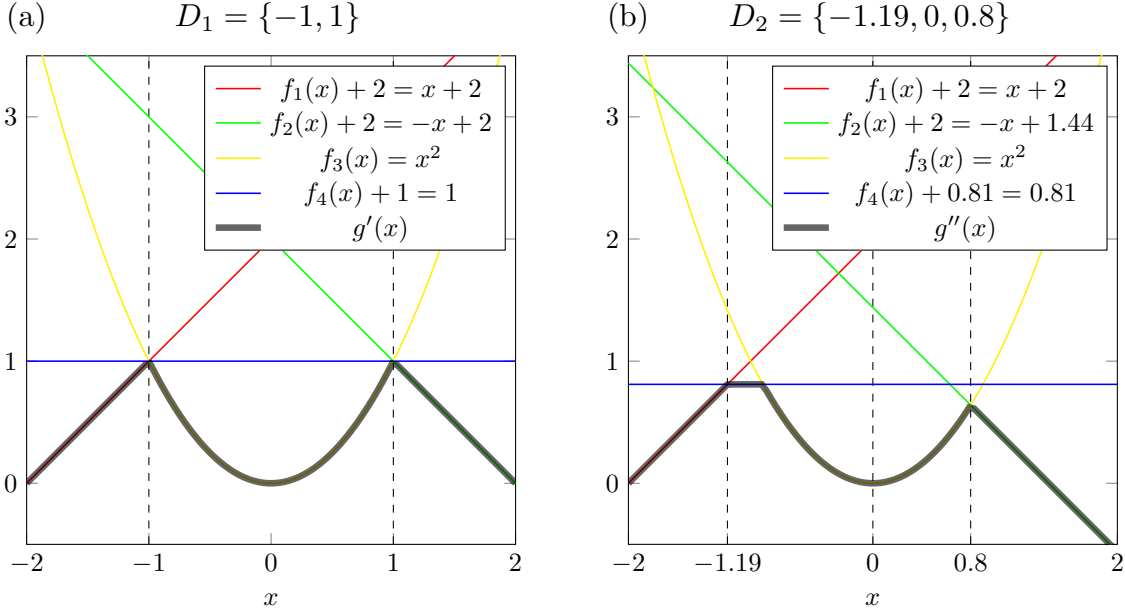
\begin{figure}
\centering
\begin{tikzpicture}[
  declare function={
    g(\x)= (\x<=-1) * (\x+2)   +
     and(\x>-1, \x<=1) * (\x*\x)     +
     (\x>1) * (-\x+2); 
    gg(\x)= (\x<=-1.19) * (\x+2)   + and(\x>-1.19, \x<=-0.9) * (0.81)  +
     and(\x>-0.9, \x<=0.8) * (\x*\x)     +
     (\x>0.8) * (-\x+1.44);           
  }
]

\begin{scope} [shift={(0,0)}]
\node[scale=1.2] at (0,7) {(a)};

\node[scale=1.2] at (3.2,7) {$D_1=\{-1,1\}$};

\begin{axis}[
	width=6.5cm,
	height=6.5cm,
    scale only axis,
    xmin=-2, xmax=2,
    xtick={-2,-1,0,1,2},ytick={0,1,2,3}, 
    xlabel = $x$,
    ymin=-0.5, ymax=3.5,
]

\addplot[
	dashed,
	forget plot,
    domain=-0.5:3.5,
    samples = 60,
]
({-1},
{x});

\addplot[
	dashed,
	forget plot,
    domain=-0.5:3.5,
    samples = 60,
]
({1},
{x});

\addplot [
	line width=0.6pt,
    domain=-2:2, 
    samples=100, 
    color=red,
]
{x+2};
\addlegendentry{$f_1(x)+2=x+2$}

\addplot [
	line width=0.6pt,
    domain=-2:2, 
    samples=100, 
    color=green,
]
{-x+2};
\addlegendentry{$f_2(x)+2=-x+2$}

\addplot [
	line width=0.6pt,
    domain=-2:2, 
    samples=100, 
    color=yellow,
]
{x^2};
\addlegendentry{$f_3(x)=x^2 $}

\addplot [
	line width=0.6pt,
    domain=-2:2, 
    samples=100, 
    color=blue,
]
{1};
\addlegendentry{$f_4(x)+1=1$}

\addplot [
	line width=2.5pt,
	opacity=0.6,
    domain=-2:2, 
    samples=100, 
    color=black,
]
{g(x)};
\addlegendentry{$g'(x)$}

\end{axis}
\end{scope}

\begin{scope} [shift={(8,0)}]

\node[scale=1.2] at (0,7) {(b)};

\node[scale=1.2] at (3.2,7) {$D_2=\{-1.19,0,0.8\}$};

\begin{axis}[
	width=6.5cm,
	height=6.5cm,
    scale only axis,
    xmin=-2, xmax=2,
    xtick={-2,-1.19,0,0.8,2},ytick={0,1,2,3}, 
    xlabel = $x$,
    ymin=-0.5, ymax=3.5,
]

\addplot[
	dashed,
	forget plot,
    domain=-0.5:3.5,
    samples = 60,
]
({-1.19},
{x});

\addplot[
	dashed,
	forget plot,
    domain=-0.5:3.5,
    samples = 60,
]
({0},
{x});

\addplot[
	dashed,
	forget plot,
    domain=-0.5:3.5,
    samples = 60,
]
({0.8},
{x});

\addplot [
	line width=0.6pt,
    domain=-2:2, 
    samples=100, 
    color=red,
]
{x+2};
\addlegendentry{$f_1(x)+2=x+2$}

\addplot [
	line width=0.6pt,
    domain=-2:2, 
    samples=100, 
    color=green,
]
{-x+1.44};
\addlegendentry{$f_2(x)+2=-x+1.44$}

\addplot [
	line width=0.6pt,
    domain=-2:2, 
    samples=100, 
    color=yellow,
]
{x^2};
\addlegendentry{$f_3(x)=x^2 $}

\addplot [
	line width=0.6pt,
    domain=-2:2, 
    samples=100, 
    color=blue,
]
{0.81};
\addlegendentry{$f_4(x)+0.81=0.81$}

\addplot [
	line width=2.5pt,
	opacity=0.6,
    domain=-2:2, 
    samples=100, 
    color=black,
]
{gg(x)};
\addlegendentry{$g''(x)$}

\end{axis}
\end{scope}
\end{tikzpicture}
\caption{Examples of restricted min-plus normalization of the min-plus combination $g(x)=\min(f_1(x)+2,f_2(x)+2,f_3(x),f_4(x)+1.5)$ of functions $f_1(x)=x$, $f_2(x)=-x$, $f_3(x)=x^2$  and $f_4(x)=0$ (as in Example~\ref{E:normalization} and Figure~\ref{F:normalization}).  (a) Normalization restricted to $D_1=\{-1,1\}$; (b) Normalization restricted to $D_2=\{-1.19,0,0.8\}$.} \label{F:restricted-normalization}
\end{figure}

Reconsider the min-plus combination $g(x)=\min(f_1+2,f_2+2,f_3,f_4(x)+1.5)$ of functions $f_1(x)=x$, $f_2(x)=-x$, $f_3(x)=x^2$  and $f_4(x)=0$ as in Example~\ref{E:normalization}. Consider two distinct finite sets $D_1=\{-1,1\}$ and $D_2=\{-1.19,0,0.8\}$. We will show that the results of the normalization restricted to $D_1$ and $D_2$ are distinct. 
\begin{enumerate}[(i)]
\item Running the min-plus normalization algorithm restricted to $D_1$, we derive $\max_{x\in D_1}(g(x)-f_1(x))=\max_{x\in D_1}(g(x)-f_2(x))=2$, $\max_{x\in D_1}(g(x)-f_3(x))=0$, and $\max_{x\in D_1}(g(x)-f_4(x))=1$. Therefore, after this restricted normalization, we get the function $g'=\min(f_1+2,f_2+2,f_3,f_4+1)$ which is exactly same as the function $g$ (Figure~\ref{F:restricted-normalization}(a)).
\item Running the min-plus normalization algorithm restricted to $D_2$, we derive $\max_{x\in D_2}(g(x)-f_1(x))=2$, $\max_{x\in D_2}(g(x)-f_2(x))=1.44$, $\max_{x\in D_2}(g(x)-f_3(x))=0$, and $\max_{x\in D_2}(g(x)-f_4(x))=0.81$. Therefore, after this restricted normalization, we get the function $g''=\min(f_1+2,f_2+1.44,f_3,f_4+0.81)$ which is different from the function $g$ (Figure~\ref{F:restricted-normalization}(b)). However, one can observe that $g''\leq g$ and restricted to $x\in D_2$, we still have $g(x)=g''(x)$ (see Proposition~\ref{A:restricted-normalization}(c)).
\end{enumerate}
\end{example}

\begin{remark}
We emphasize here that by Proposition~\ref{A:restricted-normalization}, even though the derived function $g'_i$ (resp. $h'_i$) is in general not identical to the functions $g_i$ (resp. $h_i$), we still have the agreement of $g'_i$ with $g_i$ (resp. $h'_i$ with $h_i$) restricted to $D$ as demonstrated in Example~\ref{E:restricted-normalization}, which also means that the layer-to-layer independence of normalization stated in Remark~\ref{R:independence} is still valid when restricted to $D$. This is actually satisfactory enough for most of the cases we are interested, since typically the set $D$ is chosen from the training set. 
\end{remark}

\bibliographystyle{unsrt}

\bibliography{ref}

\begin{thebibliography}{10}

\bibitem{HZRS2015}
Kaiming He, Xiangyu Zhang, Shaoqing Ren, and Jian Sun.
\newblock Delving deep into rectifiers: Surpassing human-level performance on
  imagenet classification.
\newblock In {\em Proceedings of the IEEE international conference on computer
  vision}, pages 1026--1034, 2015.

\bibitem{MS15}
Diane Maclagan and Bernd Sturmfels.
\newblock {\em Introduction to tropical geometry}, volume 161.
\newblock American Mathematical Soc., 2015.

\bibitem{DB2011}
Olivier Delalleau and Yoshua Bengio.
\newblock Shallow vs. deep sum-product networks.
\newblock In {\em Advances in neural information processing systems}, pages
  666--674, 2011.

\bibitem{Bengio2011}
Yoshua Bengio and Olivier Delalleau.
\newblock On the expressive power of deep architectures.
\newblock In {\em Algorithmic Learning Theory: 22nd International Conference,
  ALT 2011, Espoo, Finland, October 5-7, 2011, Proceedings}, volume 6925,
  page~18. Springer, 2011.

\bibitem{Eldan2016}
Ronen Eldan and Ohad Shamir.
\newblock The power of depth for feedforward neural networks.
\newblock In {\em Conference on learning theory}, pages 907--940, 2016.

\bibitem{Montufar2014}
Guido~F Montufar, Razvan Pascanu, Kyunghyun Cho, and Yoshua Bengio.
\newblock On the number of linear regions of deep neural networks.
\newblock In {\em Advances in neural information processing systems}, pages
  2924--2932, 2014.

\bibitem{ZBHRV17}
Chiyuan Zhang, Samy Bengio, Moritz Hardt, Benjamin Recht, and Oriol Vinyals.
\newblock Understanding deep learning requires rethinking generalization.
\newblock In {\em 5th International Conference on Learning Representations,
  {ICLR} 2017, Toulon, France, April 24-26, 2017, Conference Track
  Proceedings}. OpenReview.net, 2017.

\bibitem{ZNL2018}
Liwen Zhang, Gregory Naitzat, and Lek-Heng Lim.
\newblock Tropical geometry of deep neural networks.
\newblock In {\em International Conference on Machine Learning}, pages
  5824--5832, 2018.

\bibitem{CM2018}
Vasileios Charisopoulos and Petros Maragos.
\newblock A tropical approach to neural networks with piecewise linear
  activations.
\newblock {\em arXiv preprint arXiv:1805.08749}, 2018.

\bibitem{CGP2019}
Giuseppe~C Calafiore, Stephane Gaubert, and Corrado Possieri.
\newblock Log-sum-exp neural networks and posynomial models for convex and
  log-log-convex data.
\newblock {\em IEEE transactions on neural networks and learning systems},
  2019.

\bibitem{CGP2020}
Giuseppe~C Calafiore, St{\'e}phane Gaubert, and Corrado Possieri.
\newblock A universal approximation result for difference of log-sum-exp neural
  networks.
\newblock {\em IEEE Transactions on Neural Networks and Learning Systems},
  2020.

\bibitem{MT2019}
Petros Maragos and Emmanouil Theodosis.
\newblock Tropical geometry and piecewise-linear approximation of curves and
  surfaces on weighted lattices.
\newblock {\em ArXiv}, abs/1912.03891, 2019.

\bibitem{MT2020}
Petros Maragos and Emmanouil Theodosis.
\newblock Multivariate tropical regression and piecewise-linear surface
  fitting.
\newblock In {\em ICASSP 2020-2020 IEEE International Conference on Acoustics,
  Speech and Signal Processing (ICASSP)}, pages 3822--3826. IEEE, 2020.

\bibitem{SM2019}
Georgios Smyrnis and Petros Maragos.
\newblock Tropical polynomial division and neural networks.
\newblock {\em ArXiv}, abs/1911.12922, 2019.

\bibitem{SMR2020}
G.~{Smyrnis}, P.~{Maragos}, and G.~{Retsinas}.
\newblock Maxpolynomial division with application to neural network
  simplification.
\newblock In {\em ICASSP 2020 - 2020 IEEE International Conference on
  Acoustics, Speech and Signal Processing (ICASSP)}, pages 4192--4196, 2020.

\bibitem{CBD2015}
Matthieu Courbariaux, Yoshua Bengio, and Jean-Pierre David.
\newblock Binaryconnect: Training deep neural networks with binary weights
  during propagations.
\newblock In {\em Advances in neural information processing systems}, pages
  3123--3131, 2015.

\bibitem{HCSEB2016}
Itay Hubara, Matthieu Courbariaux, Daniel Soudry, Ran El-Yaniv, and Yoshua
  Bengio.
\newblock Binarized neural networks.
\newblock In {\em Advances in neural information processing systems}, pages
  4107--4115, 2016.

\bibitem{RORF2016}
Mohammad Rastegari, Vicente Ordonez, Joseph Redmon, and Ali Farhadi.
\newblock Xnor-net: Imagenet classification using binary convolutional neural
  networks.
\newblock In {\em European conference on computer vision}, pages 525--542.
  Springer, 2016.

\bibitem{ZWNZWZ2016}
Shuchang Zhou, Yuxin Wu, Zekun Ni, Xinyu Zhou, He~Wen, and Yuheng Zou.
\newblock Dorefa-net: Training low bitwidth convolutional neural networks with
  low bitwidth gradients.
\newblock {\em arXiv preprint arXiv:1606.06160}, 2016.

\bibitem{CWCSXTX2019}
Hanting Chen, Yunhe Wang, Chunjing Xu, Boxin Shi, Chao Xu, Qi~Tian, and Chang
  Xu.
\newblock Addernet: Do we really need multiplications in deep learning?
\newblock {\em arXiv preprint arXiv:1912.13200}, 2019.

\bibitem{Butkovivc10}
Peter Butkovi{\v{c}}.
\newblock {\em Max-linear systems: theory and algorithms}.
\newblock Springer Science \& Business Media, 2010.

\bibitem{DS04}
Mike Develin and Bernd Sturmfels.
\newblock Tropical convexity.
\newblock {\em Doc. Math}, 9(1-27):7--8, 2004.

\bibitem{Luo2018}
Ye~Luo.
\newblock Idempotent analysis, tropical convexity and reduced divisors.
\newblock {\em arXiv preprint arXiv:1808.01987}, 2018.

\bibitem{GWMCB2013}
Ian Goodfellow, David Warde-Farley, Mehdi Mirza, Aaron Courville, and Yoshua
  Bengio.
\newblock Maxout networks.
\newblock In {\em International conference on machine learning}, pages
  1319--1327. PMLR, 2013.

\bibitem{Georgano1967}
Georganopoulos G.
\newblock Sur l'approximation des fonctions continues par des fonctions
  lipschinitziennes.
\newblock {\em Comptes rendus hebdomadaires des s\'{e}ances de l'Acad\'{e}mie
  des sciences Serie A}, 264(7):319, 1967.

\end{thebibliography}

\end{document}